\documentclass[pmlr]{jmlr}

\jmlryear{2017} 
\jmlrworkshop{Algorithmic Learning Theory 2017} 
\editors{Steve Hanneke and Lev Reyzin}
\jmlrvolume{}

\usepackage{hyperref}

\usepackage{url}            

\usepackage{amsmath}
\usepackage{amssymb}

\usepackage{color}

\newcommand{\lr}[1]{\left(#1\right)}
\newcommand{\lrs}[1]{\left[#1\right]}
\newcommand{\lrc}[1]{\left\{#1\right\}}

\newcommand{\R}{\mathbb R}
\newcommand{\E}{\mathbb E}
\newcommand{\V}{{\textrm{Var}}}

\newcommand{\F}{{\cal F}}
\newcommand{\EL}{\E_{\rho_\lambda}\lrs{\hat L(h,S)}}
\newcommand{\ELL}{\E_{\rho_\lambda}\lrs{\hat L(h,S)^2}}
\newcommand{\VL}{{\textrm{Var}}_{\rho_\lambda}\lrs{\hat L(h,S)}}
\newcommand{\tail}{\ln \frac{2 \sqrt{n}}{\delta}}
\newcommand{\Epi}{\E_\pi\lrs{e^{-n\lambda \hat L(h,S)}}}
\newcommand{\EpiL}{\E_\pi\lrs{\hat L(h,S) e^{-n\lambda \hat L(h,S)}}}
\newcommand{\EpiLL}{\E_\pi\lrs{\hat L(h,S)^2 e^{-n\lambda \hat L(h,S)}}}

\newcommand{\IA}{{\cal I}_{0a}}
\newcommand{\IB}{{\cal I}_{ab}}
\newcommand{\IC}{{\cal I}_{b1}}

\newcommand{\Lval}{\hat L^{\text{\normalfont val}}}
\newcommand{\LCV}{L_{\text{\normalfont CV}}}

\DeclareMathOperator{\KL}{KL}
\DeclareMathOperator{\kl}{kl}

\DeclareMathOperator{\round}{round}

\renewcommand{\cite}{\citep}

\title[A Strongly Quasiconvex PAC-Bayesian Bound]{A Strongly Quasiconvex PAC-Bayesian Bound}

\author{\Name{Niklas Thiemann} \Email{niklasthiemann@gmail.com}\\
				\addr Department of Computer Science, University of Copenhagen
       \AND
       \Name{Christian Igel} \Email{igel@di.ku.dk}\\
       \addr Department of Computer Science, University of Copenhagen
				\AND
				\Name{Olivier Wintenberger} \Email{olivier.wintenberger@upmc.fr}\\
				\addr LSTA, Sorbonne Universit\'es, UPMC Universit\'e Paris 06
				\AND
				\Name{Yevgeny Seldin} \Email{seldin@di.ku.dk}\\
				\addr Department of Computer Science, University of Copenhagen
				}

\begin{document} 

\maketitle

\begin{abstract}

We propose a new PAC-Bayesian bound and a way of constructing a hypothesis space, so that the bound is convex in the posterior distribution and also convex in a trade-off parameter between empirical performance of the posterior distribution and its complexity. The complexity is measured by the Kullback-Leibler divergence to a prior. We derive an alternating procedure for minimizing the bound. We show that the bound can be rewritten as a one-dimensional function of the trade-off parameter and provide sufficient conditions under which the function has a single global minimum. When the conditions are satisfied the alternating minimization is guaranteed to converge to the global minimum of the bound. We provide experimental results demonstrating that rigorous minimization of the bound is competitive with cross-validation in tuning the trade-off between complexity and empirical performance. In all our experiments the trade-off turned to be quasiconvex even when the sufficient conditions were violated.
\end{abstract}

\section{Introduction}

PAC-Bayesian analysis, where PAC stands for the Probably Approximately Correct frequentist learning model \cite{Val84}, analyzes prediction accuracy of \emph{randomized classifiers}. A randomized classifier is a classifier defined by a distribution $\rho$ over a hypothesis class ${\cal H}$. A randomized classifier predicts by drawing a hypothesis from ${\cal H}$ according to $\rho$ and applying it to make the prediction \cite{McA98}. In many applications randomized prediction is replaced by a $\rho$-weighted majority vote \cite{GLLM09}. 

PAC-Bayesian analysis provides some of the tightest generalization bounds in statistical learning theory \cite{GLLM09}. PAC-Bayesian bounds have a form of a trade-off between empirical performance of $\rho$ and its complexity, measured by the Kullback-Leibler (KL) divergence (a.k.a.\ relative entropy) between $\rho$ and a prior distribution $\pi$. Most of PAC-Bayesian literature relies on cross-validation to tune the trade-off. Cross-validation is an extremely powerful and practical heuristic for selecting model parameters, but it can potentially be misleading \cite{KMNR97,KR99}. It is also computationally expensive, especially for computationally demanding models, such as kernel SVMs, since it requires training a large number of classifiers on almost the whole dataset. Derivation of theoretical results that would not require parameter cross-validation is a long-standing challenge for theoretical machine learning \cite{Lan05}.

The need to rely on cross-validation stems from several reasons:
\begin{itemize}
	\item Not all of the existing PAC-Bayesian bounds are convex in the posterior distribution $\rho$. For example, the most widely used PAC-Bayes-kl bound due to \citet{See02} is non-convex. This makes it hard to minimize the bound with respect to the posterior distribution. In most papers the bound is replaced by a linear trade-off between empirical error and the KL divergence and the trade-off parameter is tuned by cross-validation. 
	
	While it is possible to achieve convexity in the posterior distribution $\rho$ by introducing an additional trade-off parameter \cite{Cat07,KMH11}, we are unaware of successful attempts to tune the additional trade-off parameter through rigorous bound minimization. In practice, the alternative bounds are replaced by the same linear trade-off mentioned above and tuned by cross-validation.
	\item The second obstacle is that, in order to keep the KL divergence between the posterior and the prior tractable, the set of posterior and prior distributions is often restricted. A popular example are Gaussian posteriors and Gaussian priors \cite{LST02,McA03,Lan05}. Even if the initial bound is convex in the posterior distribution, the convexity may be broken by such a restriction or reparametrization, as it happens in the Gaussian case \cite{GLLM09}.
	\item Even though PAC-Bayesian bounds are some of the tightest, we are unaware of examples, where their tightness is sufficient to compete with cross-validation in tuning the trade-off between complexity and empirical performance.
\end{itemize}

We propose a relaxation of Seeger's PAC-Bayes-kl inequality, which we name \emph{PAC-Bayes-$\lambda$ inequality} or \emph{PAC-Bayes-$\lambda$ bound} when referring to the right hand side of the inequality. The bound is convex in the posterior distribution $\rho$ and has a convex trade-off between the empirical loss and KL divergence. The inequality is similar in spirit to the one proposed by \citet{KMH11}, but it does not restrict the form of $\rho$ and $\pi$. We provide an alternating procedure for minimizing the bound. We show that the bound can be rewritten as a continuous one-dimensional function of the trade-off parameter $\lambda$ and that under certain sufficient conditions this function is strongly quasiconvex (it has a single global minimum and no other stationary points). This guarantees convergence of alternating minimization to the global optimum.

For infinite hypothesis spaces alternating minimization can be computationally intractable or require parametrization, which can break the convexity of the bound in the posterior distribution. We get around this difficulty by constructing a finite data-dependent hypothesis space. The hypothesis space is constructed by taking $m$ subsamples of size $r$ from the training data. Each subsample is used to train a weak classifier, which is then validated on the remaining $n-r$ points, where $n$ is the sample size. We adapt our PAC-Bayesian bound and minimization procedure to this setting. Our analysis and minimization procedure work for any $m$, $r$, and any split of the data, including overlaps between training sets and overlaps between validation sets. In particular, it can also be applied to aggregate models originating from a cross-validation split of the data. However, in cross-validation the training sets are typically large (of order $n$) and validation sets and the number of models are small. While the prediction accuracy is still competitive in this setting, the highest computational advantage from our approach is achieved when the relation is reversed and the training size $r$ is taken to be small, roughly of order $d$, where $d$ is the number of features, and the number of models $m$ is taken to be large, roughly of order $n$. The construction of hypothesis space can be seen as sample compression \cite{LM07}. However, unlike the common approach to sample compression, which considers all possible subsamples of a given size and thus computationally and statistically inefficient, we consider only a small subset of possible subsamples.

We provide experimental results on several UCI datasets showing that the prediction accuracy of our learning procedure (training $m$ weak classifiers and weighting their predictions through minimization of the PAC-Bayes-$\lambda$ bound) is comparable to prediction accuracy of kernel SVMs tuned by cross-validation. In addition, we show that when $r$ is considerably smaller than $n$ and $m$ is of order $n$, the comparable prediction accuracy is achieved at a much lower computation cost. The computational speed-up is achieved because of the super-quadratic training time of kernel SVMs, which makes it much faster to train many weak SVMs on small training sets than one powerful SVM on a big training set. 

In the following, we provide a brief review of PAC-Bayesian analysis, then present the PAC-Bayesian bound and its minimization procedure in Section~\ref{sec:bound}, derive conditions for convergence of minimization procedure to the global minimum in Section~\ref{sec:quasi}, describe our construction of a hypothesis space and specialize our results to this construction in Section~\ref{sec:aggregation}, and provide experimental validation in Section~\ref{sec:experiments}.

\section{A Brief Review of PAC-Bayesian Analysis}

To set the scene we start with a brief review of PAC-Bayesian analysis. 

\subsubsection*{Notations} 

We consider a supervised learning setting with an input space ${\cal X}$ and an output space ${\cal Y}$. We let $S = \lrc{(X_1,Y_1),\dots,(X_n,Y_n)}$ denote an independent identically distributed (i.i.d.) sample of size $n$ drawn according to an unknown distribution ${\cal D}(X,Y)$. A hypothesis $h$ is a function from the input to the output space $h:{\cal X} \to {\cal Y}$. We use ${\cal H}$ to denote a hypothesis class. We let $\ell:{\cal Y}^2 \to [0,1]$ denote a bounded loss function. The loss of $h$ on a sample $(X,Y)$ is $\ell(h(X),Y)$ and the expected loss of $h$ is denoted by $L(h) = \E\lrs{\ell(h(X),Y)}$. We use $\hat L(h,S) = \frac{1}{n} \sum_{i=1}^n \ell(h(X_i),Y_i)$ to denote the empirical loss of $h$ on $S$. 

A randomized prediction rule parametrized by a distribution $\rho$ over ${\cal H}$ is 
defined in the following way. For each prediction on a sample point $X$ the rule 
draws a new hypothesis $h \in {\cal H}$ according to $\rho$ and applies it to $X$. The expected loss of such prediction rule 
is 
$
\E_{h \sim \rho}\lrs{L(h)}$ and the empirical loss is 
$
\E_{h\sim\rho}\lrs{\hat L(h,S)}$. We use $\KL(\rho\|\pi) = \E_{h\sim\rho}\lrs{\ln\frac{\rho(h)}{\pi(h)}}$ to denote the KL divergence between $\rho$ and $\pi$. For Bernoulli distributions with biases $p$ and $q$ we use $\kl(p\|q)$ as a shorthand for $\KL([p,1-p]\|[q,1-q])$, the KL divergence between the two distributions. Finally, we use $\E_\rho\lrs{\cdot}$ as a shorthand for $\E_{h\sim\rho}\lrs{\cdot}$ and $\E_S\lrs{\cdot}$ as a shorthand for $\E_{S\sim{\cal D}^n}\lrs{\cdot}$.

\subsubsection*{Change of Measure Inequality}

The majority of PAC-Bayesian bounds are based on the following lemma.

\begin{lemma}[Change of Measure Inequality] For any function $f: {\cal H} \times \lr{{\cal X} \times {\cal Y}}^n \to \R$ and for any distribution $\pi$ over ${\cal H}$ , such that $\pi$ is independent of $S$, with probability greater than $1-\delta$ over a random draw of $S$, for all distributions $\rho$ over ${\cal H}$ simultaneously:
\begin{equation}
\E_{h\sim\rho}\lrs{f(h,S)} \leq \KL(\rho\|\pi) + \ln \frac{1}{\delta} + \ln \E_{h\sim\pi}\lrs{\E_{S'}\lrs{e^{f(h,S')}}}.
\label{eq:com}
\end{equation}
\label{lem:com}
\end{lemma}

The lemma is based on Donsker-Varadhan's variational definition of the KL divergence \cite{DV75}, by which $\KL(\rho\|\pi) = \sup_f \lrc{\E_\rho[f(h)] + \ln \E_\pi \lrs{e^{f(h)}}}$, where the supremum is over all measurable functions $f:{\cal H}\to\R$. In the lemma, $f$ is extended to be a function of $h$ and $S$ and then Markov's inequality is used to bound the expectation with respect to $\pi$ by $\E_\pi \lrs{e^{f(h,S)}} \leq \frac{1}{\delta} \E_{S'}\lrs{\E_\pi \lrs{e^{f(h,S')}}}$ with probability at least $1-\delta$. Independence of $\pi$ and $S$ allows to exchange the order of expectations, leading to the statement of the lemma. For a formal proof we refer to \citet{TS13}.

\subsubsection*{PAC-Bayes-kl Inequality}

Various choices of the function $f$ in Lemma~\ref{lem:com} lead to various forms of PAC-Bayesian bounds \cite{SLCB+12}. The classical choice is $f(h,S) = n \kl(\hat L(h,S)\|L(h))$. The moment generating function of $f$ can be bounded in this case by $\E_S\lrs{e^{f(h,S)}} \leq 2 \sqrt n$ \cite{Mau04,GLL+15}. This bound is used to control the last term in equation \eqref{eq:com}, leading to the PAC-Bayes-kl inequality \cite{See02}.

\begin{theorem}[PAC-Bayes-kl Inequality] For any probability distribution $\pi$ over ${\cal H}$ that is independent of $S$ and any $\delta \in (0,1)$, with probability greater than $1-\delta$ over a random draw of a sample $S$, for all distributions $\rho$ over ${\cal H}$ simultaneously:
\begin{equation}
\label{eq:PBkl}
\kl\lr{\E_\rho\lrs{\hat L(h,S)}\middle\|\E_\rho\lrs{L(h)}} \leq \frac{\KL(\rho\|\pi) + \ln \frac{2 \sqrt n}{\delta}}{n}.
\end{equation}
\label{thm:PBkl}
\end{theorem}

\section{PAC-Bayes-$\lambda$ inequality and its Alternating Minimization}
\label{sec:bound}

By inversion of the $\kl$ with respect to its second argument, inequality \eqref{eq:PBkl} provides a bound on $\E_\rho\lrs{L(h)}$. However, this bound is not convex in $\rho$ and, therefore, inconvenient for minimization. We introduce a relaxed form of the inequality, which has an additional trade-off parameter $\lambda$. The inequality leads to a bound, which is convex in $\rho$ for a fixed $\lambda$ and convex in $\lambda$ for a fixed $\rho$, making it amenable to alternating minimization. Theorem~\ref{thm:PBlambda} is analogous to \citet[Theorem 1]{KMH11} and a similar result can also be derived by using the techniques from \citet{TS13}, as shown by \citet{Thi16}.
\begin{theorem}[PAC-Bayes-$\lambda$ Inequality] For any probability distribution $\pi$ over ${\cal H}$ that is independent of $S$ and any $\delta \in (0,1)$, with probability greater than $1-\delta$ over a random draw of a sample $S$, for all distributions $\rho$ over ${\cal H}$ and $\lambda \in (0,2)$ simultaneously:
\begin{equation}
\label{eq:PBlambda}
\E_\rho\lrs{L(h)} \leq \frac{\E_\rho\lrs{\hat L(h,S)}}{1 - \frac{\lambda}{2}} + \frac{\KL(\rho\|\pi) + \ln \frac{2 \sqrt n}{\delta}}{\lambda\lr{1-\frac{\lambda}{2}}n}.
\end{equation}
\label{thm:PBlambda}
\end{theorem}
We emphasize that the theorem holds for \emph{all} values of $\lambda \in (0,2)$ simultaneously. This is in contrast to some other parametrized PAC-Bayesian bounds, for example, the one proposed by \citet{Cat07}, which hold for a \emph{fixed} value of a trade-off parameter.
\begin{proof}
We use the following analog of Pinsker's inequality \citep[Lemma 8.4]{Mar96,Mar97,Sam00,BLM13}: for $p < q$
\begin{equation}
\label{eq:kl}
\kl(p\|q) \geq (q-p)^2/(2q).
\end{equation}
By application of inequality \eqref{eq:kl}, inequality \eqref{eq:PBkl} can be relaxed to
\begin{equation}
\E_\rho\lrs{L(h)} - \E_\rho\lrs{\hat L(h,S)} \leq \sqrt{2 \E_\rho\lrs{L(h)} \frac{\KL(\rho\|\pi) + \ln \frac{2 \sqrt n}{\delta}}{n}}
\label{eq:PBsqrt}
\end{equation}
\cite{McA03}. By using the inequality $\sqrt{xy} \leq \frac{1}{2}\lr{\lambda x + \frac{y}{\lambda}}$ for all $\lambda > 0$, we have that with probability at least $1-\delta$ for all $\rho$ and $\lambda > 0$
\begin{equation}
\E_\rho\lrs{L(h)} - \E_\rho\lrs{\hat L(h,S)} \leq \frac{\lambda}{2}\E_\rho\lrs{L(h)} + \frac{\KL(\rho\|\pi) + \ln \frac{2 \sqrt n}{\delta}}{\lambda n}
\end{equation}
\cite{KMH11}. By changing sides
\[
\lr{1 - \frac{\lambda}{2}} \E_\rho\lrs{L(h)} \leq \E_\rho\lrs{\hat L(h,S)} + \frac{\KL(\rho\|\pi) + \ln \frac{2 \sqrt n}{\delta}}{\lambda n}.
\]
For $\lambda < 2$ we can divide both sides by $\lr{1-\frac{\lambda}{2}}$ and obtain the theorem statement.
\end{proof}

Since $\E_\rho\lrs{\hat L(h,S)}$ is linear in $\rho$ and $\KL(\rho\|\pi)$ is convex in $\rho$ \cite{CT06}, for a fixed $\lambda$ the right hand side of inequality \eqref{eq:PBlambda} is convex in $\rho$ and the minimum is achieved by 
\begin{equation}
\label{eq:rho}
\rho_\lambda(h) = \frac{\pi(h) e^{-\lambda n \hat L(h,S)}}{\E_\pi\lrs{e^{-\lambda n \hat L(h',S)}}},
\end{equation}
where $\E_\pi\lrs{e^{-\lambda n \hat L(h',S)}}$, a shorthand for $\E_{h'\sim\pi}\lrs{e^{-\lambda n \hat L(h',S)}}$, is a convenient way of writing the normalization factor, which covers continuous and discrete hypothesis spaces in a unified notation. Furthermore, for $t \in (0,1)$ and $c_1,c_2 \geq 0$ the function $\frac{c_1}{1-t} + \frac{c_2}{t(1-t)}$ is convex in $t$ \cite{TS13}.  Therefore, for a fixed $\rho$ the right hand side of inequality \eqref{eq:PBlambda} is convex in $\lambda$ for $\lambda \in (0,2)$. The minimum is achieved by 
\begin{equation}
\label{eq:lambda}
\lambda = \frac{2}{\sqrt{\frac{2n \E_\rho\lrs{\hat L(h,S)}}{\lr{\KL(\rho\|\pi) + \ln \frac{2 \sqrt n}{\delta}}} + 1} + 1}
\end{equation}
\citep{TS13}. Note that the optimal value of $\lambda$ is smaller than 1 and that for $n \geq 4$ it can be lower bounded as $\lambda \geq \frac{2}{\sqrt{2 n + 1} + 1} \geq \frac{1}{\sqrt n}$. Alternating application of the update rules \eqref{eq:rho} and \eqref{eq:lambda} monotonously decreases the bound, and thus converges to a local minimum. 

Unfortunately, the bound is \emph{not} jointly convex in $\rho$ and $\lambda$ (this can be verified by computing the Hessian of the first term and taking large $n$, so that the second term can be ignored). Joint convexity would have been a sufficient condition for convergence to the global minimum, but it is \emph{not} a necessary condition. In the following section we show that under certain conditions alternating minimization is still guaranteed to converge to the global minimum of the bound despite absence of joint convexity. 

\section{Strong Quasiconvexity of the PAC-Bayes-$\lambda$ Bound}
\label{sec:quasi}

We denote the right hand side of the bound in Theorem~\ref{thm:PBlambda} by
\[
\F(\rho,\lambda) = \frac{\E_\rho\lrs{\hat L(h,S)}}{1 - \lambda/2} + \frac{\KL(\rho\|\pi) + \ln \frac{2 \sqrt{n}}{\delta}}{n \lambda(1-\lambda/2)}.
\]
By substituting the optimal value of $\rho$ from equation \eqref{eq:rho} into $\F(\rho,\lambda)$ and applying the identity
\begin{align}
\KL(\rho_\lambda\|\pi) &= \E_{\rho_\lambda}\lrs{\ln \frac{\rho_\lambda(h)}{\pi(h)}} = \E_{\rho_\lambda}\lrs{\ln \frac{e^{-n\lambda \hat L(h,S)}}{\E_\pi\lrs{e^{-n\lambda \hat L(h',S)}}}}\notag\\ &= -n\lambda \EL - \ln\Epi
\label{eq:KL}
\end{align}
we can write $\F$ as a function of a single scalar parameter $\lambda$:
\begin{align*}
\F(\lambda) &= \frac{\E_{\rho_\lambda}\lrs{\hat L(h,S)}}{1 - \lambda/2} + \frac{\KL(\rho_\lambda\|\pi) + \ln \frac{2 \sqrt{n}}{\delta}}{n \lambda(1-\lambda/2)}\\
&= \frac{\E_{\rho_\lambda}\lrs{\hat L(h,S)}}{1 - \lambda/2} - \frac{\E_{\rho_\lambda}\lrs{\hat L(h,S)}}{1 - \lambda/2} + \frac{-\ln \Epi + \ln \frac{2 \sqrt{n}}{\delta}}{n \lambda(1-\lambda/2)}\\
&= \frac{- \ln \Epi + \ln \frac{2 \sqrt{n}}{\delta}}{n \lambda(1-\lambda/2)}.
\end{align*}

We note that $\F(\lambda)$ is not necessarily convex in $\lambda$. For example, taking ${\cal H} = \lrc{h_1,h_2}$, $\hat L(h_1,S) = 0$, $\hat L(h_2,S) = 0.5$, $\pi(h_1) = \pi(h_2) = \frac{1}{2}$, $n = 100$, and $\delta = 0.01$ produces a non-convex $\F$. However, we show that $\F(\lambda)$ is strongly quasiconvex under a certain condition on the variance defined by
\[
\VL = \ELL - \EL^2.
\]
We recall that a univariate function $f:\mathcal{I}\to\R$ defined on an interval $\mathcal{I} \subseteq \R$ is \emph{strongly quasiconvex} if for any $x,y \in \mathcal{I}$ and $t \in (0,1)$ we have $f(tx + (1-t)y) < \max\lrc{f(x),f(y)}$.
\begin{theorem}
$\F(\lambda)$ is continuous and if at least one of the two conditions
\begin{equation}
2\KL(\rho_\lambda\|\pi) + \ln\frac{4n}{\delta^2} > \lambda^2 n^2 \VL \label{eq:KL-cond}
\end{equation}
or
\begin{equation}
\EL > (1-\lambda) n \VL \label{eq:E-cond}
\end{equation}
is satisfied for all $\lambda \in \lrs{\sqrt{\frac{\tail}{n}}, 1}$, then $\F(\lambda)$ is strongly quasiconvex for $\lambda \in (0,1]$ and alternating application of the update rules \eqref{eq:rho} and \eqref{eq:lambda} converges to the global minimum of $\F$.
\label{thm:global-min}
\end{theorem}

\begin{proof}
The proof is based on inspection of the first and second derivative of $\F(\lambda)$. Calculation of the derivatives is provided in Appendix~\ref{app:derivatives}. The existence of the first derivative ensures continuity of $\F(\lambda)$. By inspecting the first derivative we obtain that stationary points of $\F(\lambda)$ corresponding to $\F'(\lambda) = 0$ are characterized by the identity
\[
2 (1 - \lambda) \lr{\KL(\rho_\lambda\|\pi) + \ln \frac{2\sqrt n}{\delta}} = \lambda^2 n \EL.
\]
The identity provides a lower bound on the value of $\lambda$ at potential stationary points. Using the facts that $\EL \leq 1$ and for $\lambda \leq \frac{1}{2}$ the complement $(1-\lambda) \geq \frac{1}{2}$, for $n \geq 7$ we have
\[
\lambda = \sqrt{\frac{2 (1 - \lambda) \lr{\KL(\rho_\lambda\|\pi) + \ln \frac{2\sqrt n}{\delta}}}{n\EL}} \geq \min\lr{\sqrt{\frac{\KL(\rho_\lambda\|\pi)+\tail}{n}}~,~\frac{1}{2}} \geq \sqrt{\frac{\tail}{n}}.
\]
By expressing $\KL(\rho_\lambda\|\pi)$ via $\EL$ (or the other way around) and substituting it into the second derivative of $\F(\lambda)$ we obtain that if either of the two conditions of the theorem is satisfied at a stationary point then the second derivative of $\F(\lambda)$ is positive there. Thus, if \eqref{eq:KL-cond} or \eqref{eq:E-cond} is satisfied for all $\lambda \in \lrs{\sqrt{\frac{\tail}{n}}, 1}$ then all stationary points of $\F(\lambda)$ (if any exist) are local minima. Since $\F(\lambda)$ is a continuous one-dimensional function it means that $\F(\lambda)$ is strongly quasiconvex (it has a single global minimum and no other stationary points). Since alternating minimization monotonously decreases the value of $\F(\lambda)$ it is guaranteed to converge to the global minimum.
\end{proof}

Next we show a sufficient condition for \eqref{eq:KL-cond} to be satisfied for a finite ${\cal H}$ for all $\lambda$.
\begin{theorem}
Let $m$ be the number of hypotheses in ${\cal H}$ and assume that the prior $\pi(h)$ is uniform. Let $a = \frac{\sqrt{\ln \frac{4n}{\delta^2}}}{n\sqrt 3}$, $b = \frac{\ln(3mn^2)}{\sqrt{n\tail}}$, and $K = \frac{e^2}{12}\ln\frac{4n}{\delta^2}$. Let $x_h = \hat L(h,S) - \min_h \hat L(h,S)$. If the number of hypotheses for which $x_h \in (a,b)$ is at most $K$ then $\VL \leq \frac{\ln \frac{4n}{\delta^2}}{\lambda^2 n^2}$ for all $\lambda \in \lrs{\sqrt{\frac{\tail}{n}}, 1}$ and $\F(\lambda)$ is strongly quasiconvex and its global minimum is guaranteed to be found by alternating application of the update rules \eqref{eq:rho} and \eqref{eq:lambda}.
\label{thm:var}
\end{theorem}

The theorem splits hypotheses in ${\cal H}$ into ``good'', ``mediocre'', and ``bad''. ``Good'' hypotheses are those for which $x_h \leq a$, meaning that the empirical loss $\hat L(h,S)$ is close to the best. ``Mediocre'' are those for which $x_h \in (a,b)$. ``Bad'' are those for which $x_h \geq b$.  The theorem states that as long as the number of ``mediocre'' hypotheses is not too large, $\F(\lambda)$ is guaranteed to be quasiconvex.

\begin{proof}
We have $\VL = \V_{\rho_\lambda}\lrs{\hat L(h,S) - \min_h \hat L(h,S)}$. Under the assumption of a uniform prior $\rho_\lambda(h) = e^{-n\lambda x_h} / \sum_{h'} e^{-n\lambda x_{h'}}$. Since for $h^* = \arg\min_h \hat L(h,S)$ we have $x_{h^*} = 0$, the denominator satisfies $\sum_h e^{-n \lambda x_h} \geq 1$. Let $\IA = [0,a]$, $\IB = (a,b)$, and $\IC = [b,1]$ be the intervals corresponding to ``good'', ``mediocre'', and ``bad'' hypotheses. We have:
\begin{align*}
\VL &\leq \E_{\rho_\lambda}\lrs{x_h^2}\\
&= \frac{\sum_h x_h^2 e^{-n\lambda x_h}}{\sum_h e^{-n\lambda x_h}}\\
&= \frac{\sum_{x_h\in \IA} x_h^2 e^{-n\lambda x_h}}{\sum_h e^{-n\lambda x_h}} + \frac{\sum_{x_h\in \IB} x_h^2 e^{-n\lambda x_h}}{\sum_h e^{-n\lambda x_h}} + \frac{\sum_{x_h\in \IC} x_h^2 e^{-n\lambda x_h}}{\sum_h e^{-n\lambda x_h}}\\
&\leq a^2 + \sum_{x_h\in \IB} x_h^2 e^{-n\lambda x_h} + \sum_{x_h\in \IC} x_h^2 e^{-n\lambda x_h}.
\end{align*}
We show a number of properties of the above expression. First, we recall that $\lambda \leq 1$. Therefore,
\[
a^2 = \frac{\ln \frac{4n}{\delta^2}}{3n^2} \leq \frac{\ln\frac{4n}{\delta^2}}{3\lambda^2 n^2}.
\]
For the second term, simple calculus gives $x^2 e^{-n\lambda x} \leq \frac{4}{e^2 n^2 \lambda^2}$. Since by the theorem assumption there are at most $K$ hypotheses falling in $\IB$,
\[
\sum_{h\in \IB} x_h^2 e^{-n\lambda x_h} \leq \frac{4K}{e^2 n^2 \lambda^2} \leq \frac{\ln\frac{4n}{\delta^2}}{3\lambda^2 n^2}.
\]
For the last term we have
\[
b > \frac{2}{\sqrt{n\tail}} = \frac{2}{n}\sqrt{\frac{n}{\tail}} \geq \frac{2}{\lambda n}
\]
and we note that for $x \geq 2 / \lambda n$ the function $x^2 e^{-n\lambda x}$ is monotonically decreasing in $x$. Since $\lambda \geq \sqrt{\frac{\tail}{n}}$ we obtain
\begin{equation}
\label{eq:b}
\sum_{x_h \in \IC} x_h^2 e^{-n\lambda x_h} \leq m b^2 e^{-n\lambda b} \leq m b^2 e^{-\sqrt{n\tail} b} \leq m e^{-\sqrt{n\tail} b} = m \frac{1}{3m n^2} = \frac{1}{3n^2} \leq \frac{\ln\frac{4n}{\delta^2}}{3\lambda^2 n^2}.
\end{equation}
By taking all three terms together we arrive at
\[
\VL \leq \frac{\ln \frac{4n}{\delta^2}}{\lambda^2 n^2},
\]
which implies condition \eqref{eq:KL-cond} of Theorem~\ref{thm:global-min} since $\KL(\rho_\lambda\|\pi) \geq 0$.
\end{proof}
In Appendix~\ref{app:thm-var} we provide a couple of relaxations of the conditions in Theorem~\ref{thm:var}. The first allows to trade the boundaries $a$ and $b$ of the intervals with the value of $K$ and the second improves the value of $b$.

In our experiments presented in Section~\ref{sec:experiments}, $\F(\lambda)$ turned to be convex even when the sufficient conditions of Theorem~\ref{thm:var} (including the relaxations detailed in Appendix~\ref{app:thm-var}) were violated. This suggests that it may be possible to relax the conditions even further. At the same time, it is possible to construct artificial examples, where $\F(\lambda)$ is not quasiconvex. For example, taking $n = 200$, $\delta = 0.25$, and $m = \round\lr{e^{0.74 n\Delta}} + 1 \approx 2.7 \cdot 10^6$ hypotheses (where $\round$ is rounding to the nearest integer) with $\hat L(h_1, S) = 0$ and $\hat L(h_i, S) = \Delta = 0.1$ for all $i \in \lrc{2,\dots,m}$ and a uniform prior leads to $\F(\lambda)$ with two local minima. The artificial example requires $m$ to be of the order of $e^{n \lambda^\star \Delta}$, where $\lambda^\star$ is the value of $\lambda$ at a stationary point of $\F(\lambda)$ and $\Delta$ is the loss of suboptimal hypotheses (in the example $\Delta = 0.1$). Thus, quasiconvexity is not always guaranteed, but it holds in a wide range of practical scenarios.

\section{Construction of a Hypothesis Space}
\label{sec:aggregation}

Computation of the partition function (the denominator in \eqref{eq:rho}) is not always tractable, however, it can be easily computed when ${\cal H}$ is finite. 
The crucial step is to construct a sufficiently powerful finite hypothesis space ${\cal H}$. Our proposal is to construct ${\cal H}$ by training $m$ hypotheses, where each hypothesis is trained on $r$ random points from $S$ and validated on the remaining $n-r$ points. This construction resembles a cross-validation split of the data. However, in cross-validation $r$ is typically large (close to $n$) and validation sets are non-overlapping. Our approach works for any $r$ and has additional computational advantages when $r$ is small. We do not require validation sets to be non-overlapping and overlaps between training sets are allowed. Below we describe the construction more formally.

Let $h \in \lrc{1,\dots,m}$ index the hypotheses in ${\cal H}$. Let $S_h$ denote the training set of $h$ and $S\setminus S_h$ the validation set. $S_h$ is a subset of $r$ points from $S$, which are selected independently of their values (for example, subsampled randomly or picked according to a predefined partition of the data). We define the validation error of $h$ by $\Lval(h,S) = \frac{1}{n-r} \sum_{(X,Y) \in S\setminus S_h} \ell(h(X),Y)$. Note that the validation errors are $(n-r)$ i.i.d.\ random variables with bias $L(h)$ and, therefore, for $f(h,S) = (n-r) \kl(\Lval(h,S)\|L(h))$ we have $\E_S\lrs{e^{f(h,S)}} \leq 2 \sqrt{n-r}$. The following result is a straightforward adaptation of Theorem~\ref{thm:PBlambda} to our setting. A proof sketch is provided in Appendix~\ref{app:PBaggregation}.
\begin{theorem} Let $S$ be a sample of size $n$. Let ${\cal H}$ be a set of $m$ hypotheses, where each $h \in {\cal H}$ is trained on $r$ points from $S$ selected independently of the composition of $S$. For any probability distribution $\pi$ over ${\cal H}$ that is independent of $S$ and any $\delta \in (0,1)$, with probability greater than $1-\delta$  over a random draw of a sample S, for all distributions $\rho$ over $\cal H$ and $\lambda \in (0,2)$ simultaneously:
\begin{equation}
\label{eq:PBaggregation}
\E_\rho\lrs{L(h)} \leq \frac{\E_\rho\lrs{\Lval(h,S)}}{1 - \frac{\lambda}{2}} + \frac{\KL(\rho\|\pi) + \ln \frac{2 \sqrt{n-r}}{\delta}}{\lambda\lr{1-\frac{\lambda}{2}}(n-r)}.
\end{equation}
\label{thm:PBaggregation}
\end{theorem}
It is natural, but not mandatory to select a uniform prior $\pi(h) = 1/m$. The bound in equation \eqref{eq:PBaggregation} can be minimized by alternating application of the update rules in equations \eqref{eq:rho} and \eqref{eq:lambda} with $n$ being replaced by $n-r$ and $\hat L$ by $\Lval$.


\section{Experimental Results}
\label{sec:experiments}

In this section we study how PAC-Bayesian weighting of weak classifiers proposed in Section~\ref{sec:aggregation} compares with ``strong'' kernel SVMs tuned by cross-validation and trained on all training data. The experiments were performed on eight UCI datasets \citep{Asuncion+Newman:2007} summarized in Table \ref{tbl:uci_sizes}. In our experiments we employed the SVM solver from LIBSVM \citep{libsvm}.

\begin{table}
\begin{center}
\begin{tabular}{|l|c|c|c|c|c|c|c|c|}
\hline
 & Mushrooms & Skin & Waveform & Adult & Ionosphere & AvsB & Haberman & Breast\\ \hline
$|\text{S}|$ & 2000 & 2000 & 2000 & 2000 & 200 & 1000 & 150 & 340 \\ \hline
$|\text{T}|$ & 500 & 500 & 500 & 500 & 150 & 500 & 150 & 340 \\ \hline
$d$ & 112 & 3 & 40 & 122 & 34 & 16 & 3 & 10 \\ \hline
\end{tabular}
\end{center}
\caption{\textbf{Datasets summary.} $|S| = n$ refers to the size of the training set and $|T|$ refers to the size of the test set, $d$ refers to the number of features. ``Breast'' abbreviates ``Breast cancer'' dataset.}
\label{tbl:uci_sizes}
\end{table}

We compared the prediction accuracy and run time of our prediction strategy with a baseline of RBF kernel SVMs tuned by cross-validation. For the baseline we used 5-fold cross-validation for selecting the soft-margin parameter, $C$, and the bandwidth parameter $\gamma$  of the kernel $k(X_i,X_j)=\exp(-\gamma\|X_i-X_j\|^2)$. The value of $C$ was selected from a grid, such that $\log_{10}C \in \{-3,-2, \hdots, 3\}$. The values for the grid of $\gamma$-s were selected using the heuristic proposed in \citet{JDH99}. Specifically, for $i \in \{1,\hdots, n\}$ we defined
$G(X_i) = \min_{(X_j, Y_j) \in S \wedge Y_i \neq Y_j}\left\lVert X_i - X_j \right\rVert $.
We then defined a seed $\gamma_J$ by
$\gamma_{J} = \frac{1}{2\cdot\text{median}(G)^{2}}$.
Finally, we took a geometrically spaced grid around $\gamma_J$, so that $\gamma \in \{\gamma_{J}10^{-4}, \gamma_{J}10^{-2}, \hdots, \gamma_{J}10^{4}\}$.

For our approach we selected $m$ subsets of $r$ points uniformly at random from the training set $S$. We then trained an RBF kernel SVM for each subset. The kernel bandwidth parameter $\gamma$ was randomly selected for each subset from the same grid as used in the baseline. In all our experiments very small values of $r$, typically up to $d+1$ with $d$ being the input dimension, were sufficient for successfully competing with the prediction accuracy of the baseline and provided the most significant computational improvement. For such small values of $r$ it was easy to achieve perfect separation of the training points and, therefore, selection of $C$ was unnecessary.
The performance of each weak classifier was validated on $n-r$ points not used in its training. The weighting of classifiers $\rho$ was then computed through alternating minimization of the bound in Theorem~\ref{thm:PBaggregation}. 

In most of PAC-Bayesian literature it is common to replace randomized prediction with $\rho$-weighted majority vote. From a theoretical point of view the error of $\rho$-weighted majority vote is bounded by at most twice the error of the corresponding randomized classifier, however, in practice it usually performs better than randomized prediction \citep{GLLM09}. In our main experiments we have followed the standard choice of using the $\rho$-weighted majority vote. In Appendix~\ref{app:majority} we provide additional experiments showing that in our case the improvement achieved by taking the majority vote was very minor compared to randomized prediction. We use the term \emph{PAC-Bayesian aggregation} to refer to prediction with $\rho$-weighted majority vote.

\begin{figure}[t]
\centering
\subfigure[Ionosphere dataset. $\LCV = 0.06$.]{
\includegraphics[width=0.5\textwidth]{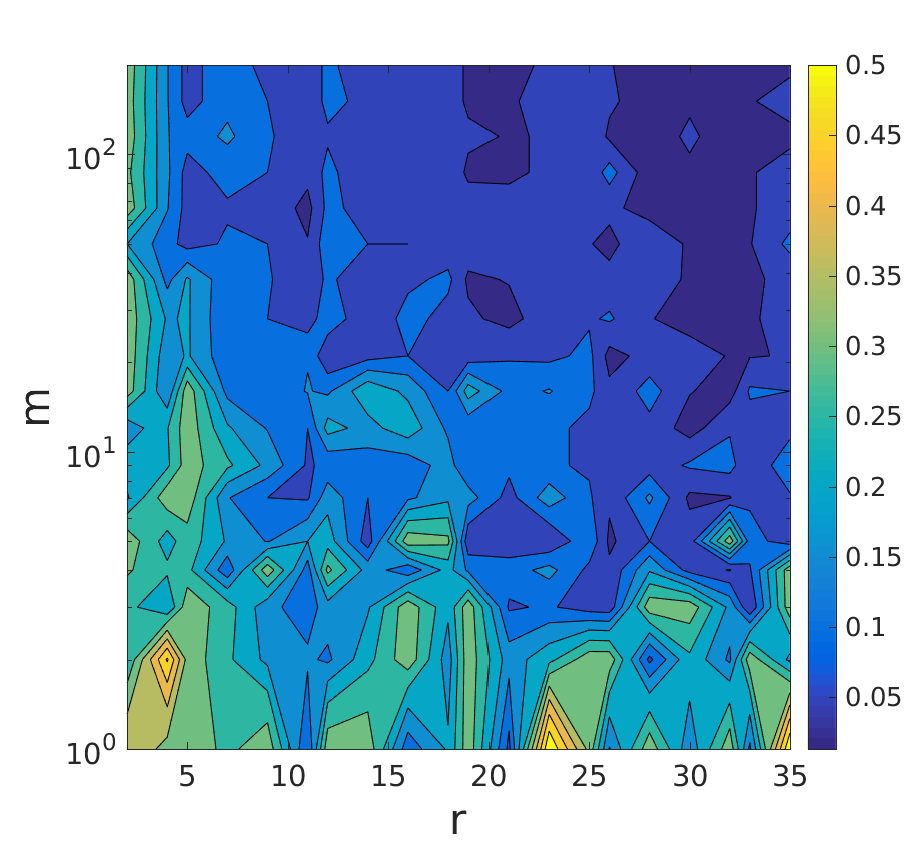}
}%
\subfigure[Mushrooms dataset. $\LCV = 0$]{
\includegraphics[width=0.5\textwidth]{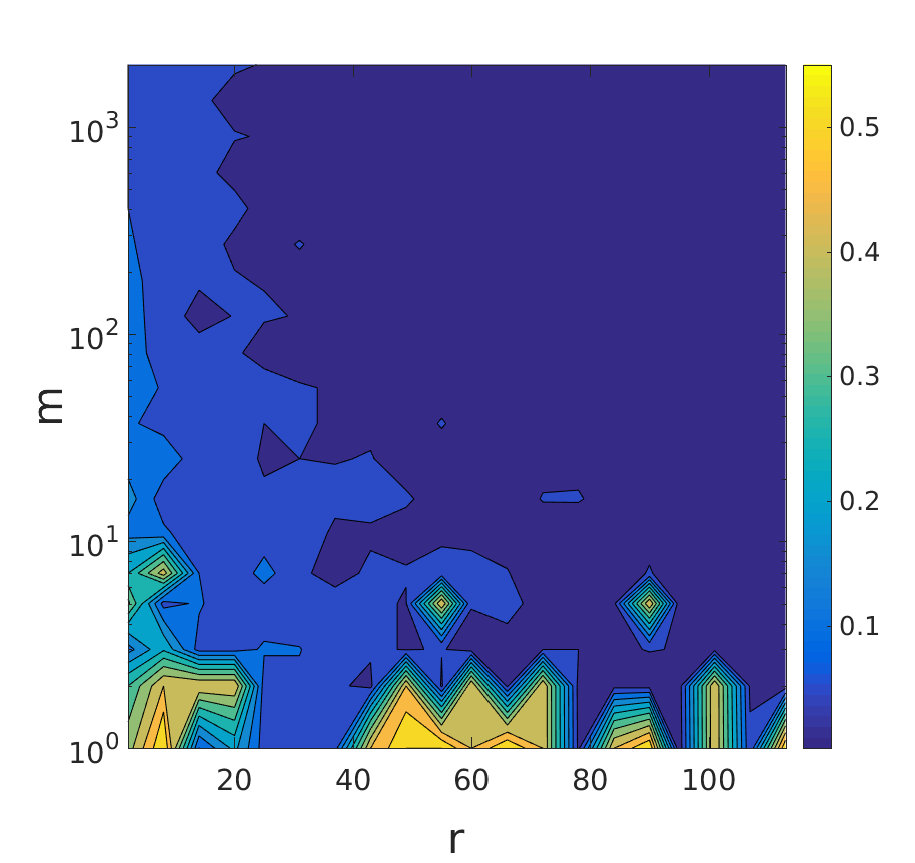}
}
\caption{\textbf{Prediction accuracy of PAC-Bayesian aggregation vs.\ cross-validated SVM across different values of $m$ and $r$.} The colors of the heatmap represent the difference between the zero-one loss of the $\rho$-weighted majority vote and the zero-one loss of the cross-validated SVM. The loss of the cross-validated SVM is given by $\LCV$ in the caption.}
\label{fig:heatmap_kernel}
\end{figure}

\begin{figure}[th!]
\centering
\subfigure[Ionosphere dataset. $n = 200$,\newline$r = d+1 = 35$.]{
\includegraphics[width=0.48\textwidth]{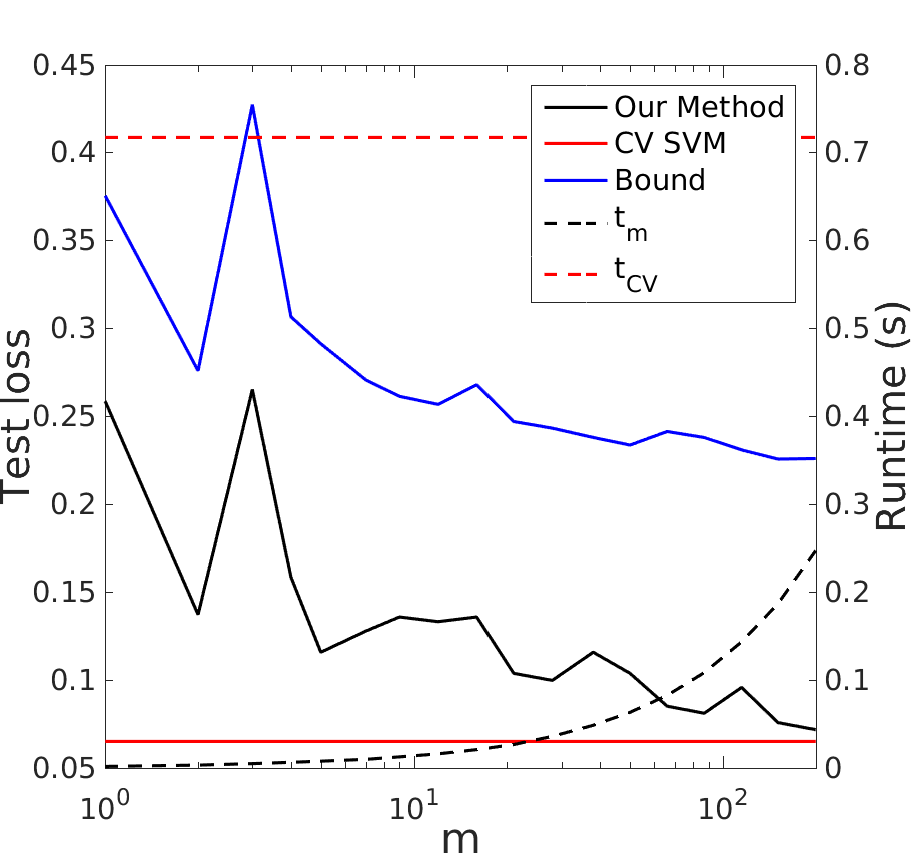}
}%
\subfigure[Waveform dataset. $n = 2000$,\newline$r = d+1 = 41$.]{
\includegraphics[width=0.48\textwidth]{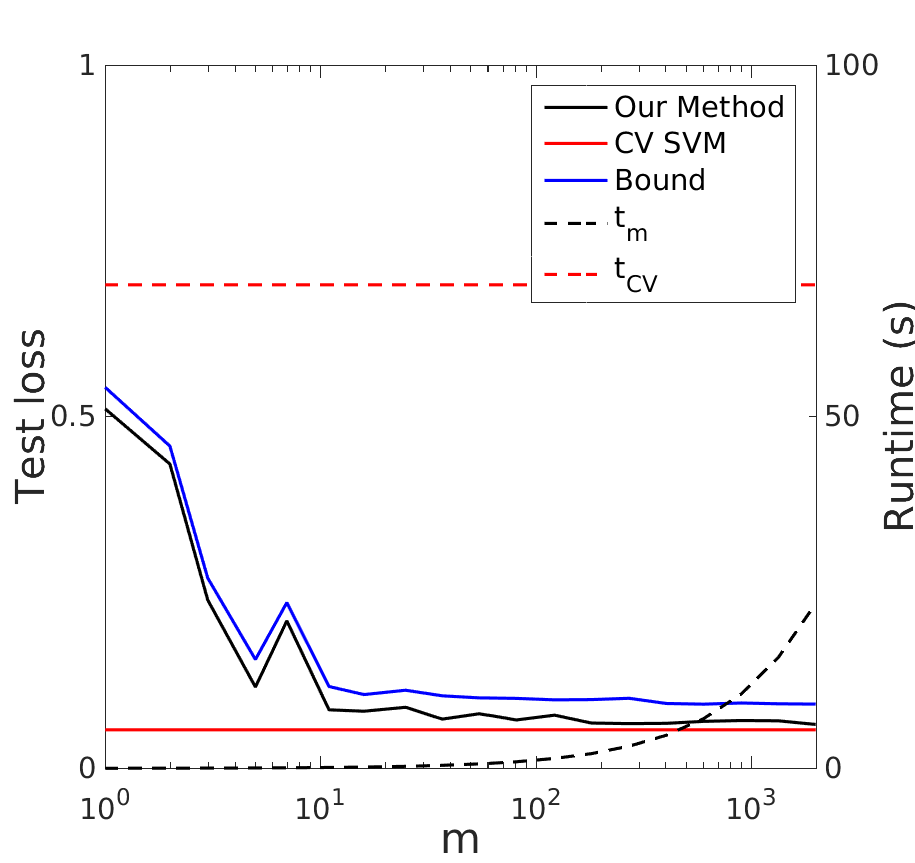}
}
\subfigure[Breast cancer dataset. $n = 340$,\newline$r = d+1 = 11$.]{
\includegraphics[width=0.48\textwidth]{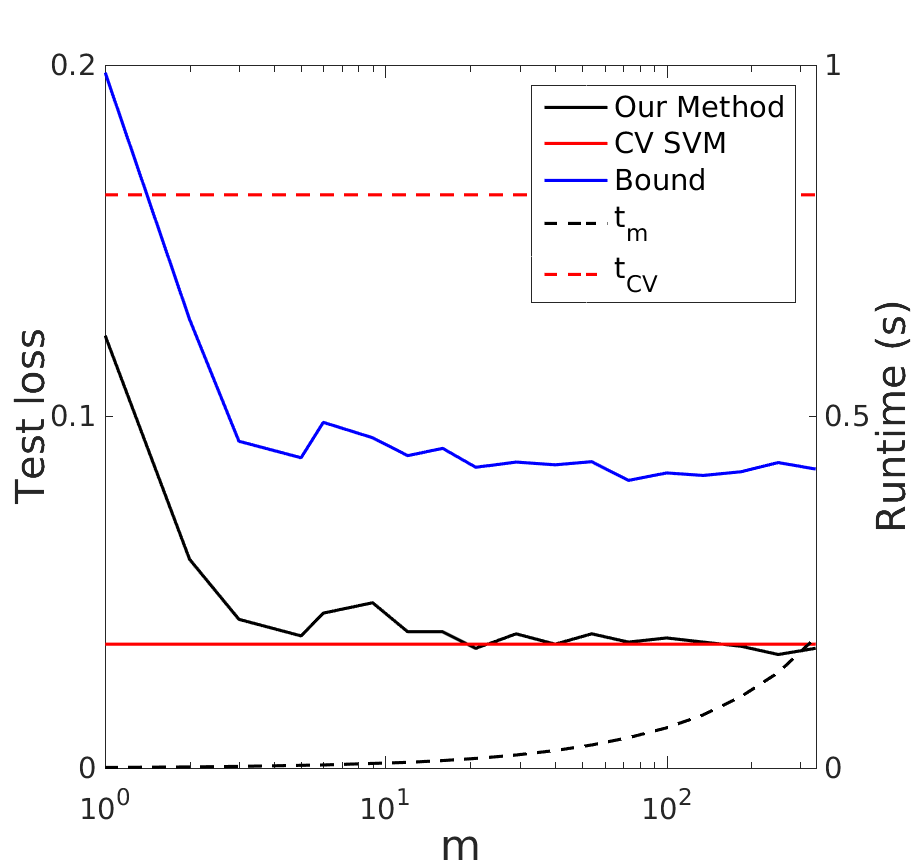}
}%
\subfigure[AvsB dataset. $n = 1000$, $r = d+1 = 17$.]{
\includegraphics[width=0.48\textwidth]{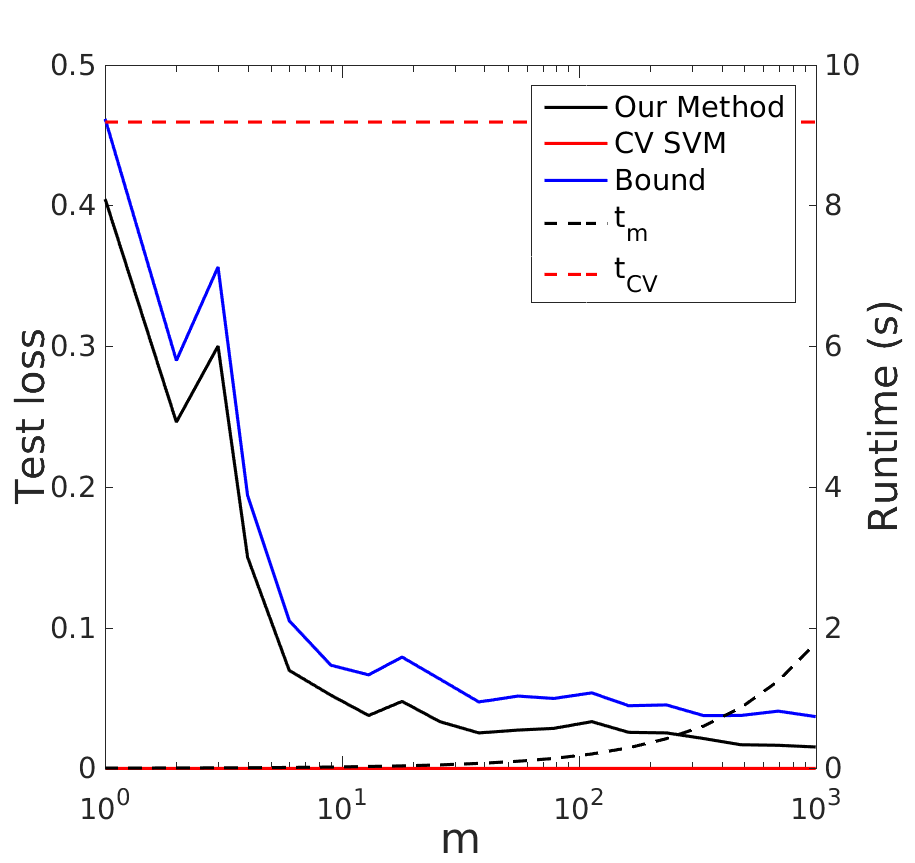}
}
\caption{\textbf{Comparison of PAC-Bayesian aggregation with RBF kernel SVM tuned by cross-validation.} The solid red, black, and blue lines correspond, respectively, to the zero-one test loss of the cross-validated SVM, the loss of $\rho$-weighted majority vote, where $\rho$ is a result of minimization of the PAC-Bayes-$\lambda$ bound, and PAC-Bayes-kl bound on the loss of randomized classifier defined by $\rho$. The dashed black line represents the training time of PAC-Bayesian aggregation, while the red dashed line represents the training time of cross-validated SVM. The prediction accuracy and run time of PAC-Bayesian aggregation and PAC-Bayes-kl bound are given as functions of the hypothesis set size $m$.}
\label{fig:runtime_kernel}
\end{figure}

In the first two experiments we studied the influence of $r$ and $m$ on classification accuracy and run time. The complexity term in Theorem~\ref{thm:PBaggregation} (the second term on the right hand side of \eqref{eq:PBaggregation}) decreases with the decrease of the training set sizes $r$ (because the size of the validation sets $n-r$ increases) and with the decrease of the number of hypotheses $m$ (because $\pi(h) = 1/m$ increases). From the computational point of view it is also desirable to have small $r$ and $m$, especially when working with expensive-to-train models, such as kernel SVMs, which have super-quadratic training time. What pushes $r$ and $m$ up is the validation error, $\E_\rho\lrs{\Lval(h,S)}$.

In the first experiment we studied the influence of $r$ and $m$ on the prediction accuracy of PAC-Bayesian aggregation. We considered 20 values of $m$ in $[1, n]$ and 20 values of $r$ in $[2, d+1]$. For each pair of $m$ and $r$ the prediction accuracy of PAC-Bayesian aggregation was evaluated, resulting in a matrix of losses. To simplify the comparison we subtracted the prediction accuracy of the baseline, thus zero values correspond to matching the accuracy of the baseline. In Figure~\ref{fig:heatmap_kernel} we show a heatmap of this matrix for two UCI datasets and the results for the remaining datasets are provided in Appendix~\ref{app:more-figures}. Overall, reasonably small values of $m$ and $r$ were sufficient for matching the accuracy of SVM tuned by cross-validation.
 
In the second experiment we provide a closer look at the effect of increasing the number $m$ of weak SVMs when their training set sizes $r$ are kept fixed. We picked $r = d+1$ and ran our training procedure with 20 values of $m$ in $[1, n]$. In Figure~\ref{fig:runtime_kernel} we present the prediction accuracy of the resulting weighted majority vote vs.\ prediction accuracy of the baseline for four datasets. The graphs for the remaining datasets are provided in Appendix~\ref{app:more-figures}. We also show the running time of our procedure vs.\ the baseline. The running time of the baseline includes cross-validation and training of the final SVM on the whole training set, while the running time of PAC-Bayesian aggregation includes training of $m$ weak SVMs, their validation, and the computation of $\rho$. In addition, we report the value of PAC-Bayes-kl bound from Theorem~\ref{thm:PBkl} on the expected loss of the randomized classifier defined by $\rho$. The kl divergence was inverted numerically to obtain a bound on the expected loss $\E_\rho\lrs{L(h)}$. The bound was adapted to our construction by replacing $n$ with $n-r$ and $\E_\rho\lrs{\hat L(h,S)}$ with $\E_\rho\lrs{\Lval(h,S)}$. We note that since the bound holds for any posterior distribution, it also holds for the distribution found by minimization of the bound in Theorem~\ref{thm:PBaggregation}. However, since Theorem~\ref{thm:PBaggregation} is a relaxation of PAC-Bayes-kl bound, using PAC-Bayes-kl for the final error estimate is slightly tighter. The bound on the loss of $\rho$-weighted majority vote is at most a factor of 2 larger than than the bound for the randomized classifier. In calculation of the bound we used $\delta = 0.05$. We conclude from the figure that relatively small values of $m$ are sufficient for matching or almost matching the prediction accuracy of the baseline, while the run time is reduced dramatically. We also note that the bound is exceptionally tight.

\begin{table}
\begin{center}
\begin{tabular}{|l|c|c|c|c|c|c|c|c|}
\hline
 & Mushrooms & Skin & Waveform & Adult & Ionosphere & AvsB & Haberman & Breast\\ \hline
$m$ & 130 & 27 & 140 & 28 & 24 & 160 & 23 & 50\\
\hline
\end{tabular}
\end{center}
\caption{\textbf{Average maximal values of $m$ for which quasiconvexity was guaranteed by Theorem~\ref{thm:var}.}}
\label{tbl:quasiconvexity}
\end{table}

\begin{figure*}
\centering
\subfigure[Breast cancer dataset with $|\text{S}| = 340$.]{
\includegraphics[width=0.29\textwidth]{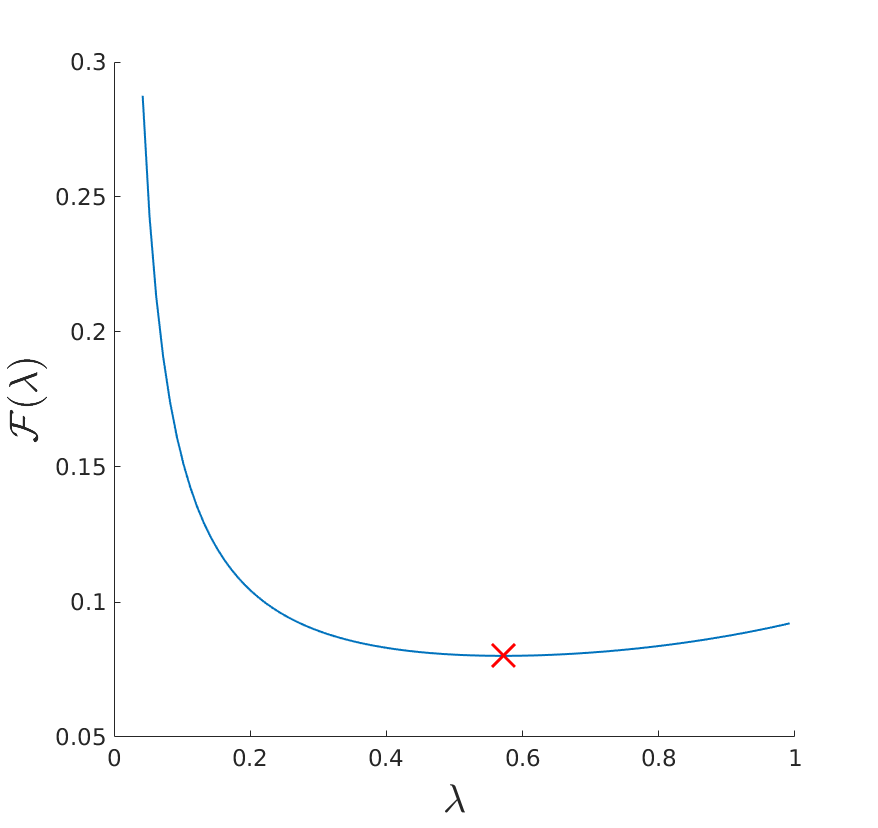}
}\quad
\subfigure[Adult dataset with $|\text{S}| = 2000$.]{
\includegraphics[width=0.29\textwidth]{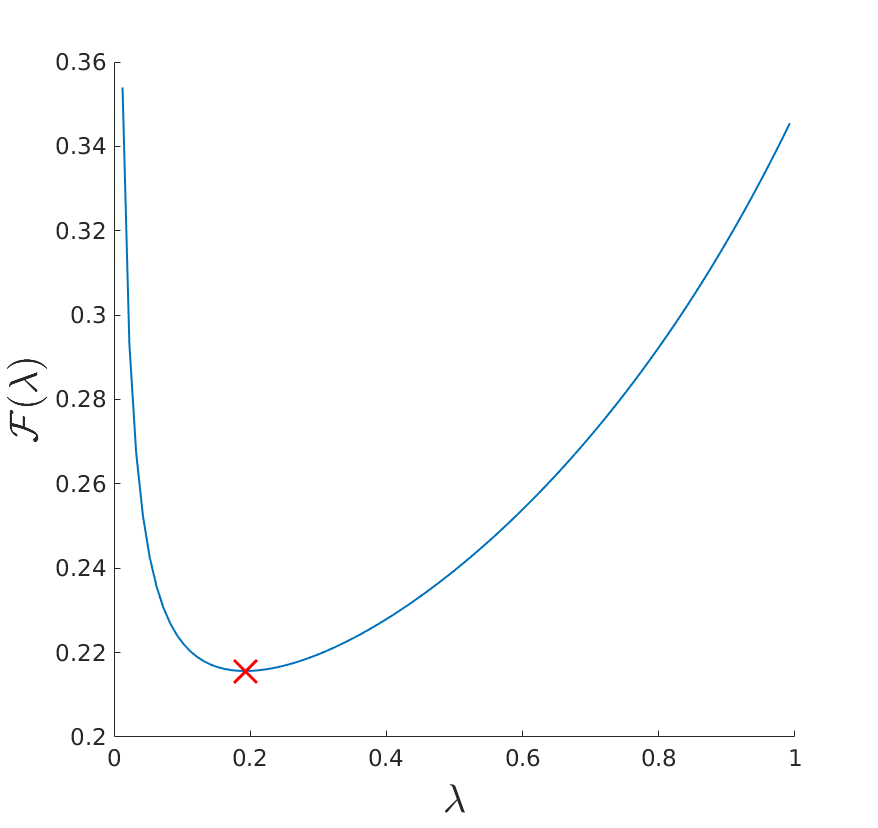}
}\quad
\subfigure[Mushrooms dataset with $|\text{S}| = 2000$.]{
\includegraphics[width=0.29\textwidth]{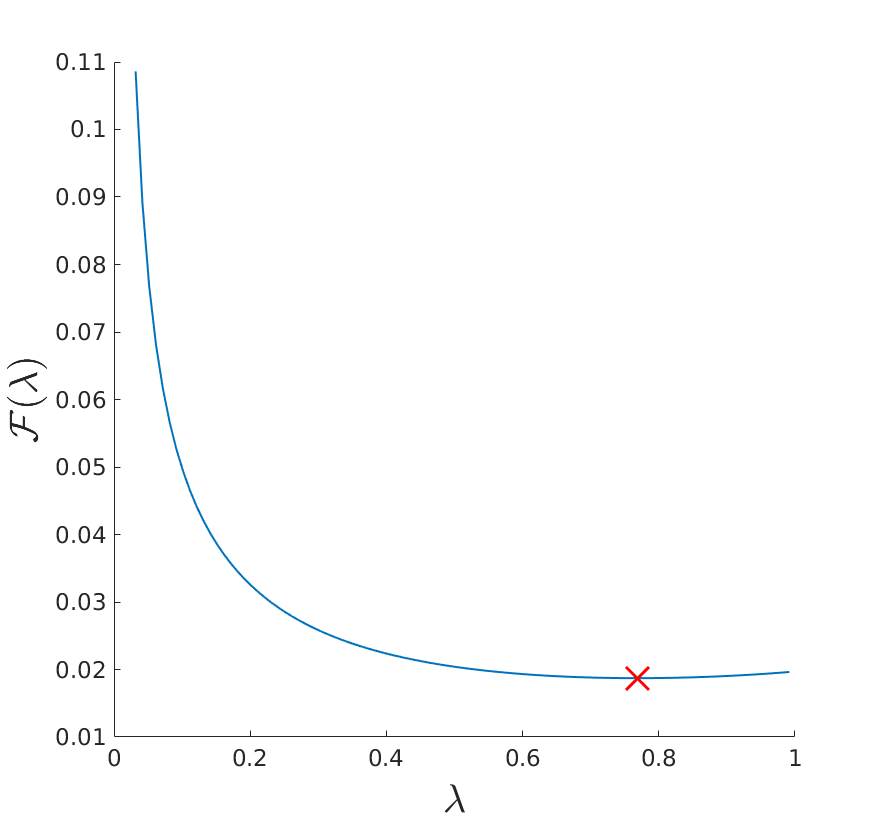}
}
\caption{\textbf{Empirical evaluation of the shape of $\F(\lambda)$}. The blue curve represents the value of $\F(\lambda)$ and the red cross marks $\lambda$ returned by alternating minimization.}
\label{fig:f_lambda_convex}
\end{figure*}

In our last experiment we tested quasiconvexity of $\F(\lambda)$. Theorem~\ref{thm:var} provided theoretical guarantee of quasiconvexity for small $m$ and numerical evaluation has further shown that $\F(\lambda)$ was convex for all values of $m$ used in our experiments. For testing the theoretical guarantees we increased $m$ in steps of 10 until the sufficient condition for strong quasiconvexity in Theorem~\ref{thm:var} was violated. The condition included adjustment of interval boundaries $a$ and $b$, as described in Appendix~\ref{app:thm-var-tune}, and improved value of $b$, as described in Appendix~\ref{app:thm-var-b}. The experiments were repeated 10 times for each dataset, where in each experiment the training sets were redrawn and a new set of hypotheses was trained, leading to a new set of validation losses $\Lval$. In Table~\ref{tbl:quasiconvexity} we report the average over the 10 repetitions of the maximal values of $m$ with guaranteed quasiconvexity. Since $\F(\lambda)$ is always quasiconvex for $m \leq K(0,0) + 1$, where $K(0,0) = \frac{e^2}{4}\ln \frac{4n}{\delta^2}$ (see equation \eqref{eq:Kab} in Appendix~\ref{app:thm-var-tune}), we report the value of $K(0,0)+1$ whenever it was not possible to ensure quasiconvexity with larger $m$. When it was not possible to guarantee quasicovexity theoretically we tested the shape of $\F(\lambda)$ empirically. Figure~\ref{fig:f_lambda_convex} shows a few typical examples. The plots in Figure~\ref{fig:f_lambda_convex} were constructed in the following way: given a sample $S$, we trained $m$ weak SVMs. We then computed the corresponding vector of validation losses, $\Lval(h, S)$. For each value in a grid of $\lambda$-s, we computed the corresponding $\rho$ according to equation \eqref{eq:rho}. Finally, we substituted the value of $\rho_\lambda$ and $\lambda$ into equation \eqref{eq:PBlambda} to get the value of the bound. In all our calculations we used a uniform prior and $\delta = 0.05$. The figure shows that $\F(\lambda)$ was convex in $\lambda$ in all the cases.

\section{Conclusion}

We have presented a new PAC-Bayesian inequality, an alternating procedure for its minimization, and a way to construct a finite hypothesis space for which the bound and minimization procedure work particularly well. We have derived sufficient conditions for the minimization procedure to converge to the global optimum of the bound. We have shown that the procedure is competitive with cross-validation in tuning the trade-off between complexity and empirical performance of $\rho$. In addition, it provides tight high-probability generalization guarantees and achieves prediction accuracies on par with kernel SVMs tuned by cross-validation, but at a considerably lower computation cost.

In our experiments the bound turned to be convex even when the sufficient conditions of Theorem~\ref{thm:var} were violated. It suggests that further relaxation of these conditions may be possible to achieve in future work.

\acks{We thank anonymous reviewers of this and earlier versions of the manuscript for valuable feedback. We also thank Oswin Krause for suggesting the use of the term ``quasiconvexity'' to describe the shape of $\F(\lambda)$. CI and YS acknowledge support by the Innovation Fund Denmark through the \emph{Danish Center for Big Data Analytics Driven Innovation} (DABAI). OW would like to thank the mathematical department of the university of Copenhagen for his guest professorship in 2015-2016.}

\bibliography{bibliography}

\begin{thebibliography}{29}
\providecommand{\natexlab}[1]{#1}
\providecommand{\url}[1]{\texttt{#1}}
\expandafter\ifx\csname urlstyle\endcsname\relax
  \providecommand{\doi}[1]{doi: #1}\else
  \providecommand{\doi}{doi: \begingroup \urlstyle{rm}\Url}\fi

\bibitem[Asuncion and Newman(2007)]{Asuncion+Newman:2007}
Arthur Asuncion and David~J. Newman.
\newblock {UCI} machine learning repository, 2007.
\newblock URL \url{www.ics.uci.edu/\~mlearn/MLRepository.html}.

\bibitem[Boucheron et~al.(2013)Boucheron, Lugosi, and Massart]{BLM13}
St{\'e}phane Boucheron, G{\'a}bor Lugosi, and Pascal Massart.
\newblock \emph{Concentration Inequalities {A} Nonasymptotic Theory of
  Independence}.
\newblock Oxford University Press, 2013.

\bibitem[Catoni(2007)]{Cat07}
Olivier Catoni.
\newblock {PAC-B}ayesian supervised classification: The thermodynamics of
  statistical learning.
\newblock \emph{IMS Lecture Notes Monograph Series}, 56, 2007.

\bibitem[Chang and Lin(2011)]{libsvm}
Chih-Chung Chang and Chih-Jen Lin.
\newblock {LIBSVM}: A library for support vector machines.
\newblock \emph{ACM Transactions on Intelligent Systems and Technology}, 2,
  2011.

\bibitem[Claesen et~al.(2014)Claesen, Smet, Suykens, and Moor]{CSSM14}
Marc Claesen, Frank~De Smet, Johan~A.K. Suykens, and Bart~De Moor.
\newblock {EnsembleSVM: A} library for ensemble learning using support vector
  machines.
\newblock \emph{Journal of Machine Learning Research}, 15\penalty0 (1), 2014.

\bibitem[Collobert et~al.(2002)Collobert, Bengio, and Bengio]{CBB02}
Ronan Collobert, Samy Bengio, and Yoshua Bengio.
\newblock A parallel mixture of {SVMs} for very large scale problems.
\newblock \emph{Neural Computation}, 14\penalty0 (5), 2002.

\bibitem[Cover and Thomas(2006)]{CT06}
Thomas~M. Cover and Joy~A. Thomas.
\newblock \emph{Elements of Information Theory}.
\newblock Wiley Series in Telecommunications and Signal Processing, 2nd
  edition, 2006.

\bibitem[Donsker and Varadhan(1975)]{DV75}
Monroe~D. Donsker and S.R.~Srinivasa Varadhan.
\newblock Asymptotic evaluation of certain {Markov} process expectations for
  large time.
\newblock \emph{Communications on Pure and Applied Mathematics}, 28, 1975.

\bibitem[Germain et~al.(2009)Germain, Lacasse, Laviolette, and
  Marchand]{GLLM09}
Pascal Germain, Alexandre Lacasse, Fran{\c c}ois Laviolette, and Mario
  Marchand.
\newblock {PAC-B}ayesian learning of linear classifiers.
\newblock In \emph{Proceedings of the International Conference on Machine
  Learning (ICML)}, 2009.

\bibitem[Germain et~al.(2015)Germain, Lacasse, Laviolette, Marchand, and
  Roy]{GLL+15}
Pascal Germain, Alexandre Lacasse, Fran{\c c}ois Laviolette, Mario Marchand,
  and Jean-Francis Roy.
\newblock Risk bounds for the majority vote: From a {PAC-Bayesian} analysis to
  a learning algorithm.
\newblock \emph{Journal of Machine Learning Research}, 16, 2015.

\bibitem[Jaakkola et~al.(1999)Jaakkola, Diekhans, and Haussler]{JDH99}
Tommi Jaakkola, Mark Diekhans, and David Haussler.
\newblock Using the fisher kernel method to detect remote protein homologies.
\newblock In \emph{In Proceedings of the Seventh International Conference on
  Intelligent Systems for Molecular Biology (ISMB)}, 1999.

\bibitem[Kearns et~al.(1997)Kearns, Mansour, Ng, and Ron]{KMNR97}
Michael Kearns, Yishay Mansour, Andrew Ng, and Dana Ron.
\newblock An experimental and theoretical comparison of model selection
  methods.
\newblock \emph{Machine Learning}, 27, 1997.

\bibitem[Kearns and Ron(1999)]{KR99}
Michael~J. Kearns and Dand Ron.
\newblock Algorithmic stability and sanity-check bounds for leave-one-out
  cross-validation.
\newblock \emph{Neural Computation}, 11, 1999.

\bibitem[Keshet et~al.(2011)Keshet, McAllester, and Hazan]{KMH11}
Joseph Keshet, David McAllester, and Tamir Hazan.
\newblock Pac-bayesian approach for minimization of phoneme error rate.
\newblock In \emph{IEEE International Conference on Acoustics, Speech, and
  Signal Processing (ICASSP)}, 2011.

\bibitem[Langford(2005)]{Lan05}
John Langford.
\newblock Tutorial on practical prediction theory for classification.
\newblock \emph{Journal of Machine Learning Research}, 6, 2005.

\bibitem[Langford and Shawe-Taylor(2002)]{LST02}
John Langford and John Shawe-Taylor.
\newblock {PAC-Bayes} \& margins.
\newblock In \emph{Advances in Neural Information Processing Systems (NIPS)},
  2002.

\bibitem[Laviolette and Marchand(2007)]{LM07}
Fran\c{c}ois Laviolette and Mario Marchand.
\newblock {PAC-Bayes} risk bounds for stochastic averages and majority votes of
  sample-compressed classifiers.
\newblock \emph{Journal of Machine Learning Research}, 8, 2007.

\bibitem[Marton(1996)]{Mar96}
Katalin Marton.
\newblock A measure concentration inequality for contracting {Markov} chains.
\newblock \emph{Geometric and Functional Analysis}, 6\penalty0 (3), 1996.

\bibitem[Marton(1997)]{Mar97}
Katalin Marton.
\newblock A measure concentration inequality for contracting {Markov} chains
  {Erratum}.
\newblock \emph{Geometric and Functional Analysis}, 7\penalty0 (3), 1997.

\bibitem[Maurer(2004)]{Mau04}
Andreas Maurer.
\newblock A note on the {PAC}-{B}ayesian theorem.
\newblock www.arxiv.org, 2004.

\bibitem[McAllester(1998)]{McA98}
David McAllester.
\newblock Some {PAC-B}ayesian theorems.
\newblock In \emph{Proceedings of the International Conference on Computational
  Learning Theory (COLT)}, 1998.

\bibitem[McAllester(2003)]{McA03}
David McAllester.
\newblock {PAC-Bayesian} stochastic model selection.
\newblock \emph{Machine Learning}, 51, 2003.

\bibitem[Samson(2000)]{Sam00}
Paul-Marie Samson.
\newblock Concentration of measure inequalities for markov chains and
  $\phi$-mixing processes.
\newblock \emph{The Annals of Probability}, 28\penalty0 (1), 2000.

\bibitem[Seeger(2002)]{See02}
Matthias Seeger.
\newblock {PAC-Bayesian} generalization error bounds for {Gaussian} process
  classification.
\newblock \emph{Journal of Machine Learning Research}, 3, 2002.

\bibitem[Seldin et~al.(2012)Seldin, Laviolette, Cesa-Bianchi, Shawe-Taylor, and
  Auer]{SLCB+12}
Yevgeny Seldin, Fran\c{c}ois Laviolette, Nicol{\`o} Cesa-Bianchi, John
  Shawe-Taylor, and Peter Auer.
\newblock {PAC-Bayesian} inequalities for martingales.
\newblock \emph{IEEE Transactions on Information Theory}, 58, 2012.

\bibitem[Thiemann(2016)]{Thi16}
Niklas Thiemann.
\newblock {PAC-Bayesian} ensemble learning.
\newblock Master's thesis, University of Copenhagen, 2016.

\bibitem[Tolstikhin and Seldin(2013)]{TS13}
Ilya Tolstikhin and Yevgeny Seldin.
\newblock {PAC-Bayes-Empirical-Bernstein} inequality.
\newblock In \emph{Advances in Neural Information Processing Systems (NIPS)},
  2013.

\bibitem[Valentini and Dietterich(2003)]{VD03}
Giorgio Valentini and Thomas~G. Dietterich.
\newblock Low bias bagged support vector machines.
\newblock In \emph{Proceedings of the International Conference on Machine
  Learning (ICML)}, 2003.

\bibitem[Valiant(1984)]{Val84}
Leslie~G. Valiant.
\newblock A theory of the learnable.
\newblock \emph{Communications of the Association for Computing Machinery}, 27,
  1984.

\end{thebibliography}

\appendix

\section{Calculation of the Derivatives of $\F(\lambda)$}
\label{app:derivatives}

We decompose $\F(\lambda)$ in the following way:
\[
\F(\lambda) = \frac{-\ln \Epi + \tail}{n\lambda(1-\lambda/2)} = f(\lambda) g(\lambda),
\]
where
\begin{align*}
f(\lambda) &= - \frac{1}{n}\ln \Epi + \frac{\tail}{n},\\
g(\lambda) &= \frac{1}{\lambda(1-\lambda/2)}.
\end{align*}
For the derivatives of $f$ and $g$ we have:
\[
f'(\lambda)  = - \frac{\frac{d}{d\lambda}\Epi}{n\Epi} = \frac{\EpiL}{\Epi}  = \EL \geq 0.
\]
\begin{align*}
f''(\lambda) &= \frac{\lr{\frac{d}{d\lambda}\EpiL}\Epi - \lr{\frac{d}{d\lambda}\Epi}\EpiL}{\Epi^2}\\
&= \frac{-n\EpiLL}{\Epi} + n \lr{\frac{\EpiL}{\Epi}}^2\\
&= -n \lr{\E_{\rho_\lambda}\lrs{\hat L(h,S)^2} - \lr{\EL}^2}\\
&=- n\V_{\rho_\lambda}\lrs{\hat L(h,S)} \leq 0.
\end{align*}
\[
g'(\lambda) = -\frac{(1 - \lambda/2 - \lambda/2)}{\lambda^2(1-\lambda/2)^2} = \frac{\lambda-1}{\lambda^2(1-\lambda/2)^2} \leq 0.
\]
\begin{align*}
g''(\lambda) &= \frac{\lambda^2(1-\lambda/2)^2 - (\lambda - 1) \lr{2 \lambda (1 - \lambda / 2)^2 - \lambda^2(1-\lambda/2)}}{\lambda^4(1-\lambda/2)^4}\\
&= \frac{\lambda(1-\lambda/2) - (\lambda - 1) \lr{2 (1 - \lambda / 2) - \lambda}}{\lambda^3(1-\lambda/2)^3}\\
&= \frac{\lambda(1-\lambda/2) - 2(\lambda - 1) (1 - \lambda)}{\lambda^3(1-\lambda/2)^3}\\
&= \frac{\lambda(1-\lambda/2) + 2(\lambda - 1)^2}{\lambda^3(1-\lambda/2)^3}\\
&= \frac{\lambda - \lambda^2/2 + 2 \lambda^2 - 4 \lambda + 2}{\lambda^3(1-\lambda/2)^3}\\
&= \frac{(3/2) \lambda^2 - 3 \lambda + 2}{\lambda^3(1-\lambda/2)^3}\\
&= \frac{3 \lambda^2 - 6 \lambda + 4}{2 \lambda^3(1-\lambda/2)^3}\\
&= \frac{3\lr{\lambda - 1}^2 + 1}{2 \lambda^3(1-\lambda/2)^3} > 0.
\end{align*}
At a stationary point we have $\F'(\lambda) = f'(\lambda)g(\lambda) + g'(\lambda)f(\lambda) = 0$. By using the identity
\begin{equation}
\label{eq:f}
f(\lambda) = \lambda \EL + \frac{\KL(\rho_\lambda\|\pi) + \tail}{n},
\end{equation}
which follows from \eqref{eq:KL}, this gives
\[
\F'(\lambda) = \frac{\EL}{\lambda(1-\lambda/2)} + \frac{(\lambda-1)\lr{\lambda \EL}}{\lambda^2(1-\lambda/2)^2} + \frac{(\lambda-1)\lr{\KL(\rho_\lambda\|\pi) + \tail}}{n \lambda^2 (1-\lambda/2)^2} = 0.
\]
This can be rewritten as
\[
\EL + \frac{(\lambda-1) \EL}{1-\lambda/2} + \frac{(\lambda-1)\lr{\KL(\rho_\lambda\|\pi) + \tail}}{n \lambda (1-\lambda/2)} = 0,
\]
\[
\frac{1}{2} \lambda \EL = \frac{(1-\lambda)\lr{\KL(\rho_\lambda\|\pi) + \tail}}{n\lambda},
\]
\[
\frac{\KL(\rho_\lambda\|\pi) + \tail}{n} = \frac{\lambda^2 \EL}{2(1-\lambda)},
\]
which characterizes the stationary points. By combining this with the identity \eqref{eq:f} we obtain that at a stationary point
\[
f(\lambda) = \lr{\lambda + \frac{\lambda^2}{2(1-\lambda)}}\EL = \frac{\lambda (1-\lambda/2)}{1-\lambda}\EL.
\]

For the second derivative we have $\F''(\lambda) = f''(\lambda) g(\lambda) + 2 f'(\lambda) g'(\lambda) + g''(\lambda) f(\lambda)$. At a stationary point
\begin{align*}
g''(\lambda)f(\lambda)+ 2f'(\lambda)g'(\lambda) &=\lr{\frac{3\lambda^2 - 6\lambda + 4}{2\lambda^3(1-\lambda/2)^3} \frac{\lambda(1-\lambda/2)}{1-\lambda} - \frac{2(1 - \lambda)}{\lambda^2 (1 - \lambda/2)^2}} \EL\\
&= \frac{\EL}{\lambda (1-\lambda/2)(1-\lambda)}.
\end{align*}
By plugging this into $\F''(\lambda)$ we obtain that at a stationary point (if such exists)
\begin{align*}
\F''(\lambda) &= \frac{-n\VL}{\lambda(1-\lambda/2)} + \frac{\EL}{\lambda(1-\lambda/2)(1-\lambda)}\\
&= \frac{1}{\lambda(1-\lambda/2)}\lr{\frac{\EL}{1-\lambda} - n \VL}.
\end{align*}
This expression is positive if
\[
\EL > (1-\lambda) n \VL
\]
or, equivalently (by characterization of a stationary point),
\[
2\KL(\rho_\lambda\|\pi) + \ln \frac{4n}{\delta^2} > \lambda^2 n^2 \VL.
\]

\section{Relaxation of the Sufficient Conditions in Theorem~\ref{thm:var}}
\label{app:thm-var}

In this section we propose a couple of relaxations of the conditions in Theorem~\ref{thm:var}. The first provides a possibility of tuning the intervals $\IA$, $\IB$, and $\IC$, and the second provides a slight improvement in the definition of $b$.

\subsection{Tuning the intervals $\IA$, $\IB$, and $\IC$}
\label{app:thm-var-tune}

We recall the definition $x_h = \hat L(h,S) - \min_h \hat L(h,S)$. In Theorem~\ref{thm:var} we have tuned $a$ and $b$ so that the contribution to $\E_{\rho_\lambda}\lrs{x_h}$ from hypotheses falling into intervals $\IA$, $\IB$, and $\IC$ is equal. Obviously, this does not have to be the case. Take $\alpha \geq 0$ and $\beta \geq 0$, such that $\alpha + \beta \leq 1$ and define 
\begin{equation}
a(\alpha) = \frac{\sqrt{\alpha \ln \frac{4n}{\delta^2}}}{n}~~~~~\text{,}~~~~~b(\beta) = \frac{\ln \lr{\frac{1}{\beta} m n^2}}{\sqrt{n\tail}}~~~~~\text{,}~~~~~K(\alpha,\beta) = \frac{e^2(1 - \alpha - \beta)}{4}\ln \frac{4n}{\delta^2}.
\label{eq:Kab}
\end{equation}
If we can find any pair $(\alpha, \beta)$, such that the number of hypotheses for which $x_h \in (a(\alpha), b(\beta))$ is at most $K(\alpha,\beta)$ then $\VL \leq \frac{\ln\frac{4n}{\delta^2}}{\lambda^2 n^2}$ for all $\lambda \in \lrs{\sqrt{\frac{\tail}{n}}, 1}$ and $\F(\lambda)$ is strongly quasiconvex. The proof is identical to the proof of Theorem~\ref{thm:var} with $\alpha$, $(1 - \alpha - \beta)$, and $\beta$ being the relative contributions to $\E_{\rho_\lambda}\lrs{x_h}$ from the intervals $\IA$, $\IB$, and $\IC$, respectively.

\subsection{Refinement of the Boundary $b$}
\label{app:thm-var-b}

In the derivation in \eqref{eq:b} we have dropped the factor $b^2$. If we would have kept it we could reduce the value of $b$. Let 
\[
b = \frac{\ln\lr{3mn \frac{4 \lr{\ln(3mn)}^2}{\ln \frac{2\sqrt n}{\delta} \ln \frac{4n}{\delta^2}}}}{\sqrt{n \ln \frac{2 \sqrt n}{\delta}}} = \frac{\ln \lr{\frac{12 mn \lr{\ln(3mn)}^2}{\ln \frac{2\sqrt n}{\delta} \ln \frac{4n}{\delta^2}}}}{\sqrt{n \ln \frac{2 \sqrt n}{\delta}}}.
\]
Assuming that $\ln \lr{\frac{12 mn \lr{\ln(3mn)}^2}{\ln \frac{2\sqrt n}{\delta} \ln \frac{4n}{\delta^2}}} \geq 2$ and $n \geq 5$ we have
\begin{align}
\sum_{x_h \in \IC} x_h^2 e^{-n\lambda x_h} &\leq m b^2 e^{-n\lambda b}\notag\\
&= m \frac{\lr{\ln(3mn) + \ln\lr{\frac{4 \lr{\ln(3mn)}^2}{\ln \frac{2 \sqrt n}{\delta} \ln \frac{4n}{\delta^2}}}}^2}{n \ln \frac{2\sqrt n}{\delta}} \frac{\ln \frac{2 \sqrt n}{\delta}\ln\frac{4n}{\delta^2}}{3mn \times 4 \lr{\ln(3mn)}^2}\notag\\
&= \frac{\ln\frac{4n}{\delta^2}}{3n^2} \frac{\lr{\ln(3mn) + \ln\lr{\frac{4 \lr{\ln(3mn)}^2}{\ln \frac{2 \sqrt n}{\delta} \ln \frac{4n}{\delta^2}}}}^2}{4\lr{\ln(3mn)}^2}\label{eq:premonster}\\
&\leq \frac{\ln\frac{4n}{\delta^2}}{3n^2}\label{eq:monster}\\
&\leq \frac{\ln\frac{4n}{\delta^2}}{3\lambda^2n^2}.\notag
\end{align}
In step \eqref{eq:monster} we used the following auxiliary calculations. For $n \geq 5$ we have $\ln \frac{2 \sqrt n}{\delta} \ln \frac{4n}{\delta^2} \geq 4$. From here, for $n \geq 5$ we have $\frac{4 \lr{\ln(3mn)}^2}{\ln \frac{2 \sqrt n}{\delta} \ln \frac{4n}{\delta^2}} \leq \lr{\ln(3mn)}^2$. Furthermore, for $x \geq 0.5$ we have $x \geq \lr{\ln x}^2$, leading to $\ln(3mn) + \ln\lr{\lr{\ln(3mn)}^2} \leq 2 \ln(3mn)$, since $3mn \geq 0.5$. Thus, the second fraction in line \eqref{eq:premonster} is bounded by 1.

\subsection{Combining the two improvements}

It is obviously possible to combine the two improvements by defining 
\[
b(\beta) = \frac{\ln \lr{\frac{\frac{4mn}{\beta}  \lr{\ln\frac{mn}{\beta}}^2}{\ln \frac{2\sqrt n}{\delta} \ln \frac{4n}{\delta^2}}}}{\sqrt{n \ln \frac{2 \sqrt n}{\delta}}}
\]
and $a(\alpha)$ and $K(\alpha,\beta)$ as before. We only have to check that the conditions $\ln \lr{\frac{\frac{4mn}{\beta} \lr{\ln\frac{mn}{\beta}}^2}{\ln \frac{2\sqrt n}{\delta} \ln \frac{4n}{\delta^2}}} \geq 2$ and $\frac{mn}{\beta} \geq 0.5$ are satisfied and, otherwise, tune further. Note that since $\hat L(h,S)$ is trivially upper bounded by 1, $b(\beta) > 1$ is vacuous.
 
\section{A Proof Sketch of Theorem~\ref{thm:PBaggregation}}
\label{app:PBaggregation}

In this section we provide a sketch of a proof of Theorem~\ref{thm:PBaggregation}. The proof is a straightforward adaptation of the proof of Theorem~\ref{thm:PBlambda}.
\begin{proof}
As we have already mentioned in the text, since the validation errors are $(n-r)$ i.i.d.\ random variables with bias $L(h)$, for $f(h,S) = (n-r) \kl(\Lval(h,S)\|L(h))$ we have $\E_S\lrs{e^{f(h,S)}} \leq 2 \sqrt{n-r}$ \cite{Mau04}. With this result replacing $f(h,S) = n \kl(\hat L(h,S)\|L(h))$ and $\E_S\lrs{e^{f(h,S)}} \leq 2 \sqrt{n}$ in the proof of Theorem~\ref{thm:PBkl} it is straightforward to obtain an analogue of Theorem~\ref{thm:PBkl}. Namely, that for any probability distribution $\pi$ over ${\cal H}$ that is independent of $S$ and any $\delta \in (0,1)$, with probability greater than $1-\delta$ over a random draw of a sample $S$, for all distributions $\rho$ over ${\cal H}$ simultaneously
\begin{equation}
\kl\lr{\E_\rho\lrs{\Lval(h,S)}\middle\|\E_\rho\lrs{L(h)}} \leq \frac{\KL(\rho\|\pi) + \ln \frac{2 \sqrt{n-r}}{\delta}}{n-r}.
\end{equation}
And from here, directly following the steps in the proof of Theorem~\ref{thm:PBlambda}, we obtain Theorem~\ref{thm:PBaggregation}.
\end{proof}

\begin{figure*}
\centering
\subfigure[Adult dataset. $L_{CV} = 0.15$]{
\includegraphics[width=0.41\textwidth]{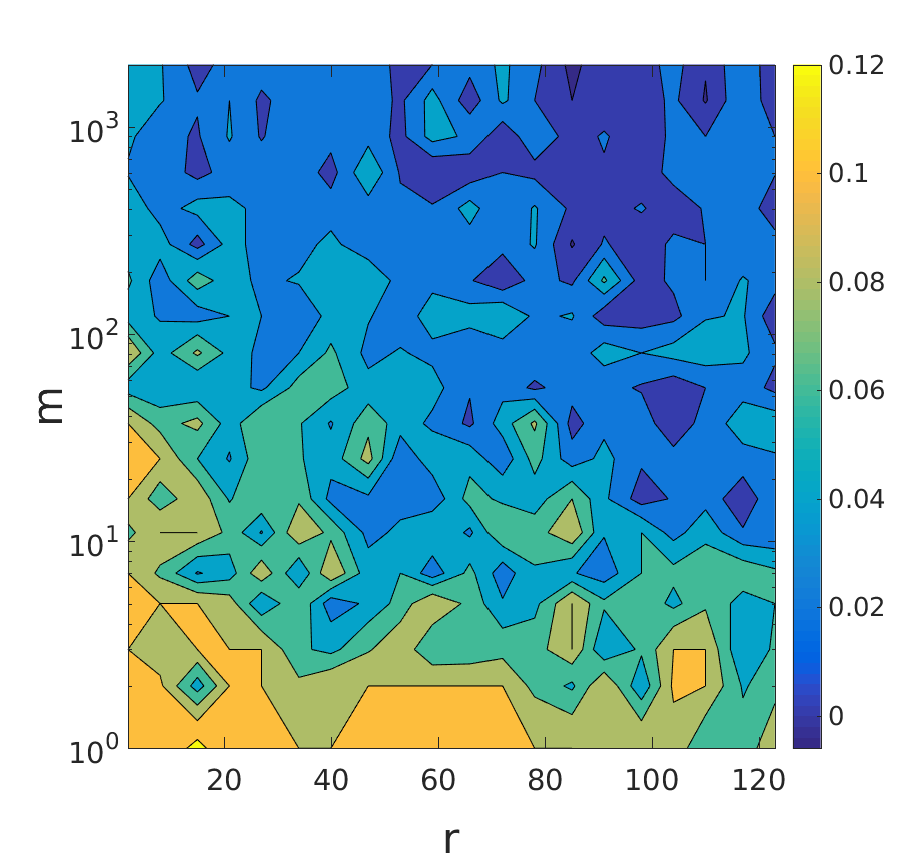}
}%
\qquad\qquad
\subfigure[Breast cancer dataset. $L_{CV} = 0.05$]{
\includegraphics[width=0.41\textwidth]{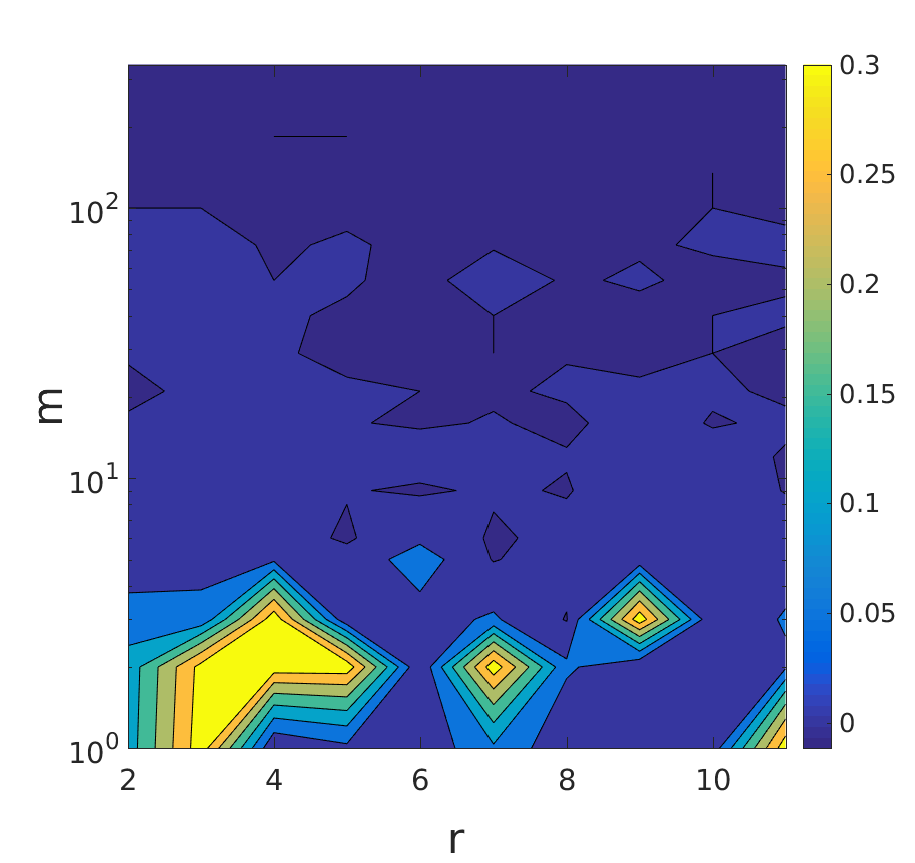}
}
\subfigure[Haberman dataset. $L_{CV} = 0.26$]{
\includegraphics[width=0.41\textwidth]{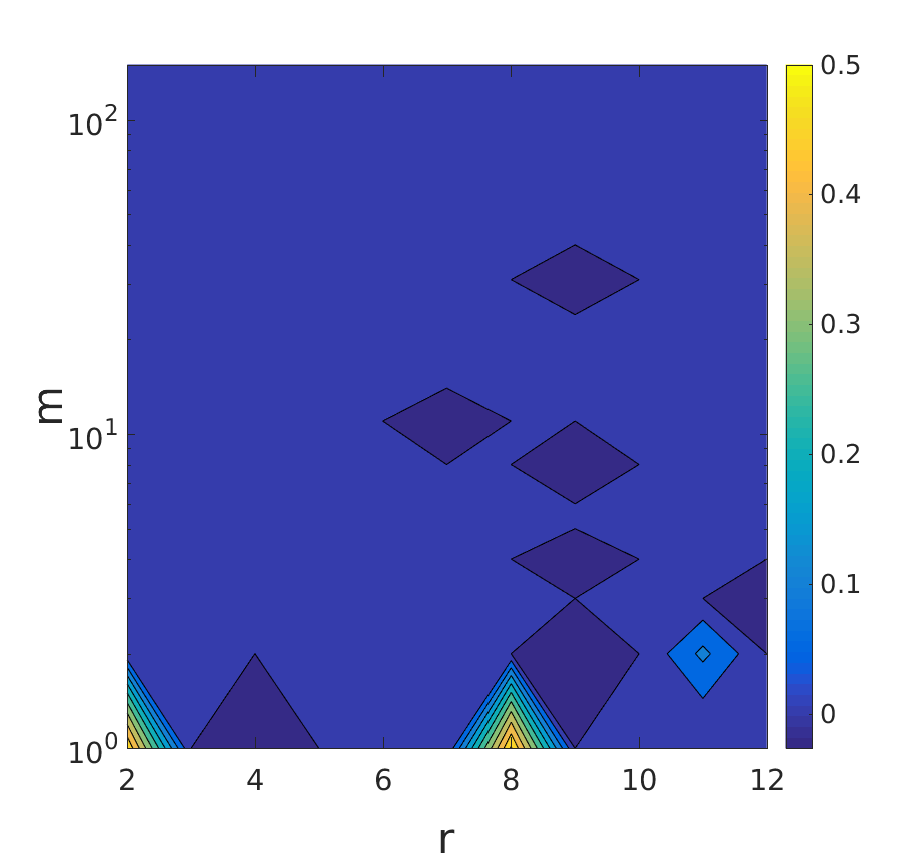}
}%
\qquad\qquad
\subfigure[AvsB dataset. $L_{CV} = 0$]{
\includegraphics[width=0.41\textwidth]{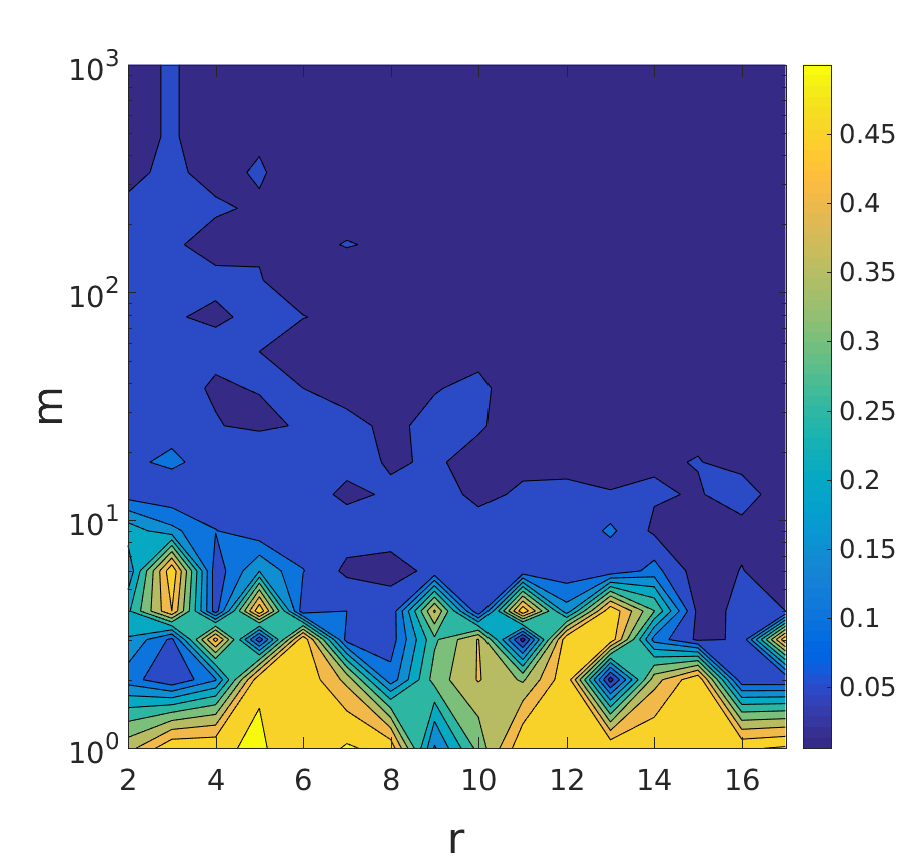}
}
\subfigure[Skin dataset. $L_{CV} = 0$]{
\includegraphics[width=0.41\textwidth]{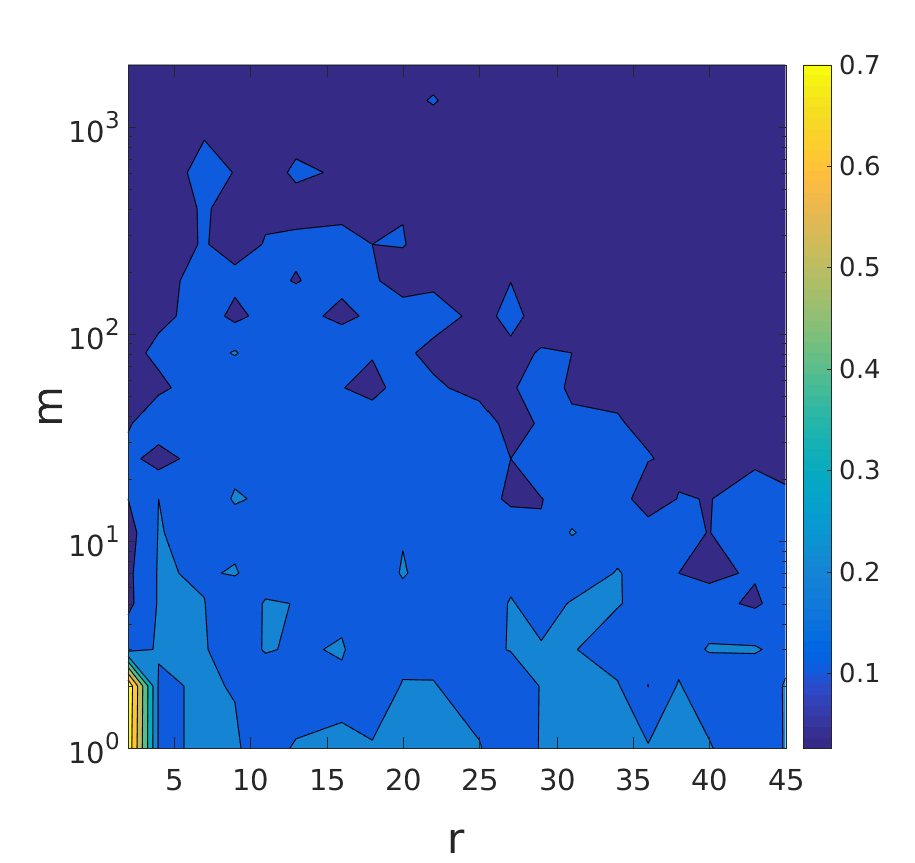}
}%
\qquad\qquad
\subfigure[Waveform dataset. $L_{CV} = 0.06$]{
\includegraphics[width=0.41\textwidth]{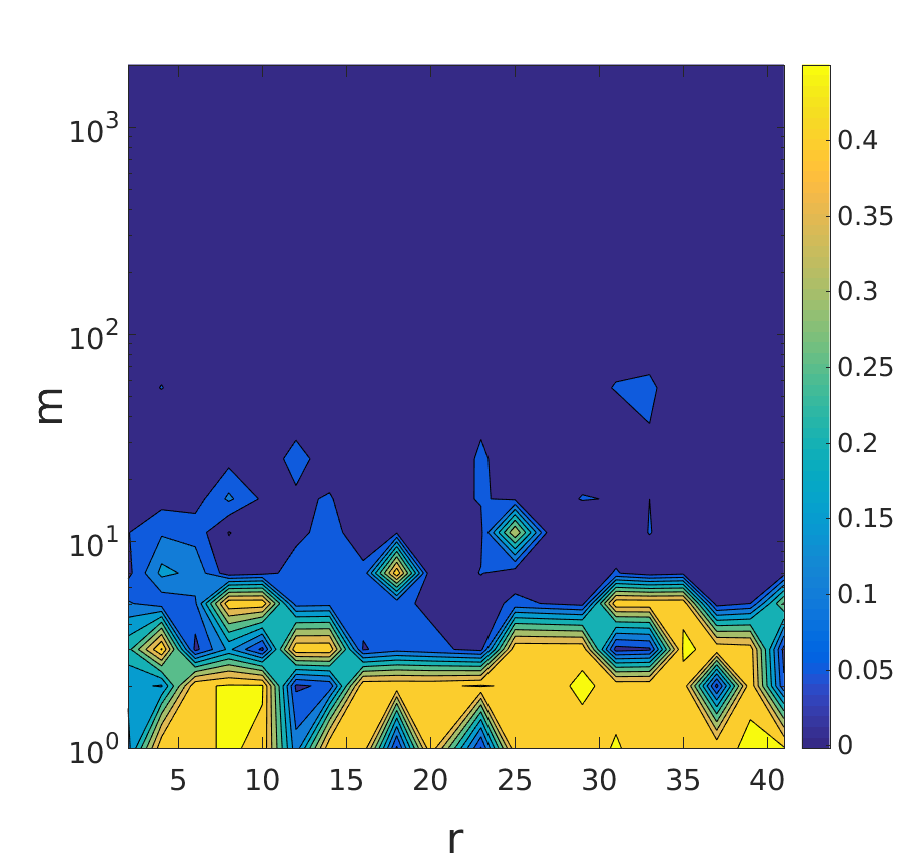}
}
\caption{\textbf{Prediction accuracy of PAC-Bayesian aggregation vs.\ cross-validated SVM across different values of $m$ and $r$.} The colors of the heatmap represent the difference between the zero-one loss of the $\rho$-weighted majority vote and the zero-one loss of the cross-validated SVM. The loss of the cross-validated SVM is given by $\LCV$ in the caption.}
\label{fig:appheatmap}
\end{figure*}

\section{Additional Experiments}
\label{app:experiments}

In this section we present figures with experimental results for UCI datasets that could not be included in the body of the paper. We also present three additional experiments.

\subsection{Additional Figures for the Main Experiments}
\label{app:more-figures}

We present the outcomes of experiments in Section \ref{sec:experiments} for additional UCI datasets. Since the \verb|skin| and \verb|Haberman| datasets have low dimensionality ($d = 3$) we use $r = \sqrt{n}$ rather than $r = d + 1$, to get a reasonable subsample size. Figure~\ref{fig:appheatmap} continues the plots in Figure~\ref{fig:heatmap_kernel}, and Figure~\ref{fig:appruntime} continues the plots in Figure~\ref{fig:runtime_kernel}.

\subsection{Comparison with Uniform Weighting and Best Performing Classifier}

In Figures \ref{fig:appprediction} and \ref{fig:appprediction_max} we compare the prediction accuracy of $\rho$-weighted majority vote with uniformly weighted majority vote, which is popular in ensemble learning \cite{CBB02,VD03, CSSM14}. As a baseline the prediction accuracy of a cross-validated SVM is also shown. For the two datasets in Figure \ref{fig:appprediction_max} we also include the prediction accuracy of SVM corresponding to the maximum value of $\rho$ (which is the best performing SVM in the set). Due to significant overlap with the weighted majority vote, the latter graph is omitted for the datasets in Figure \ref{fig:appprediction}. Overall, in our setting the accuracy of $\rho$-weighted majority vote is comparable to the accuracy of the best classifier in the set and significantly better than uniform weighting.

\begin{figure*}
\centering
\subfigure[Adult dataset. $n = 2000$, $r = d+1 = 123$.]{
\includegraphics[width=0.5\textwidth]{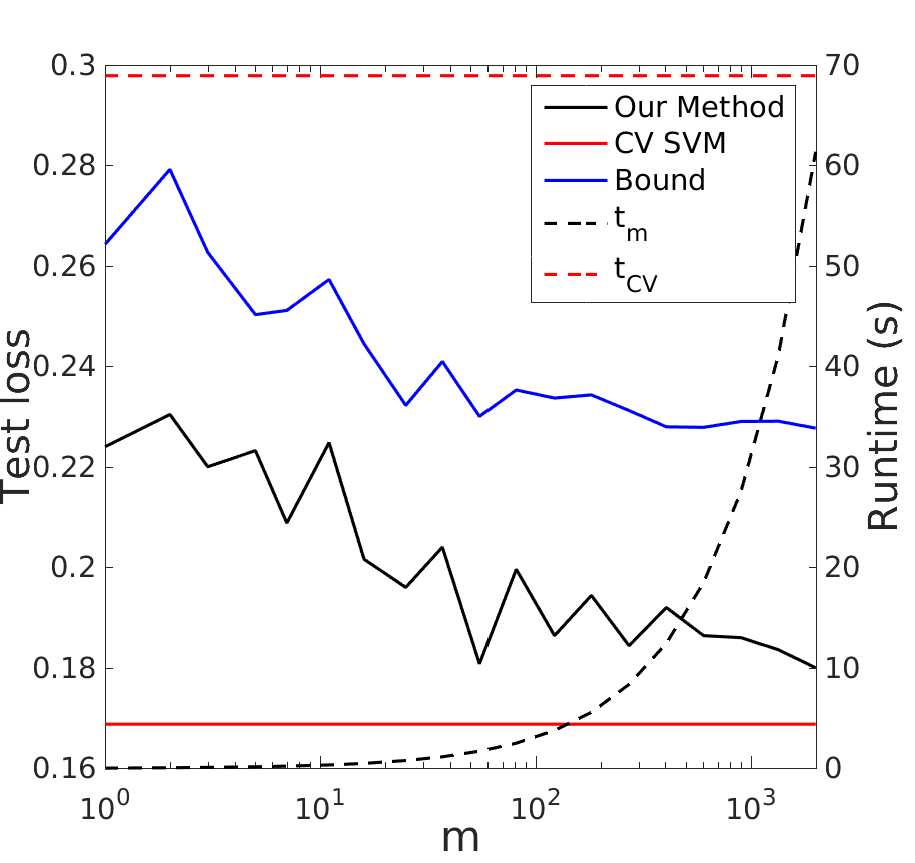}
}%
\subfigure[Haberman dataset. $n = 150$, $r = \sqrt n = 12$.]{
\includegraphics[width=0.5\textwidth]{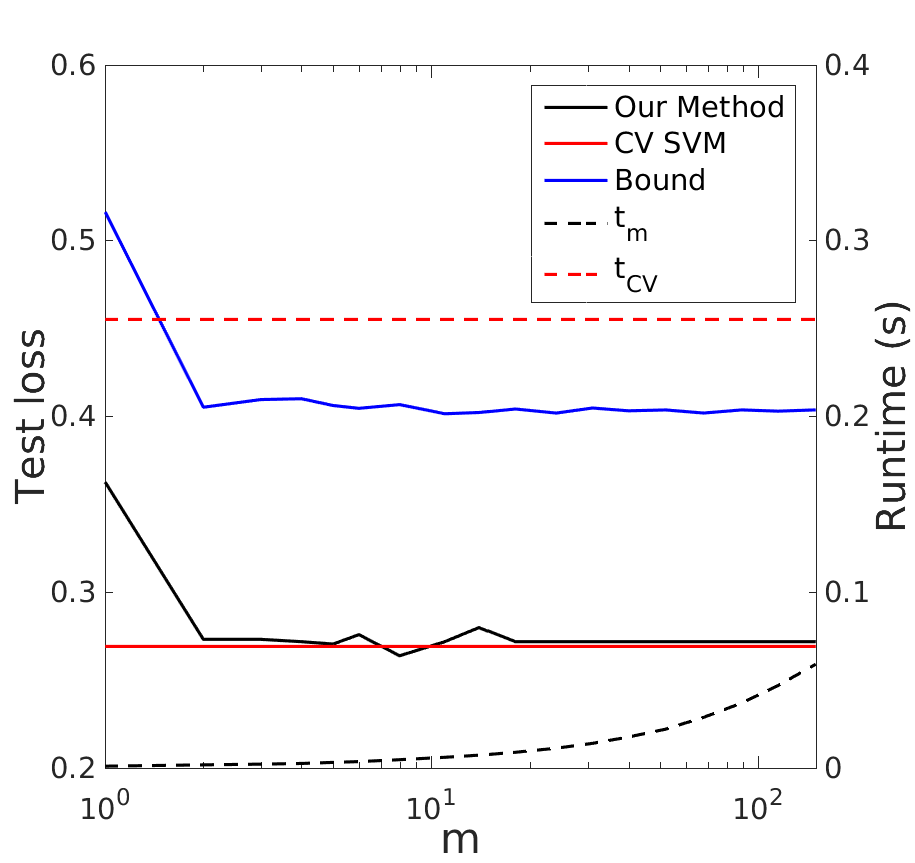}
}
\subfigure[Skin dataset. $n = 2000$, $r = \sqrt n = 45$.]{
\includegraphics[width=0.5\textwidth]{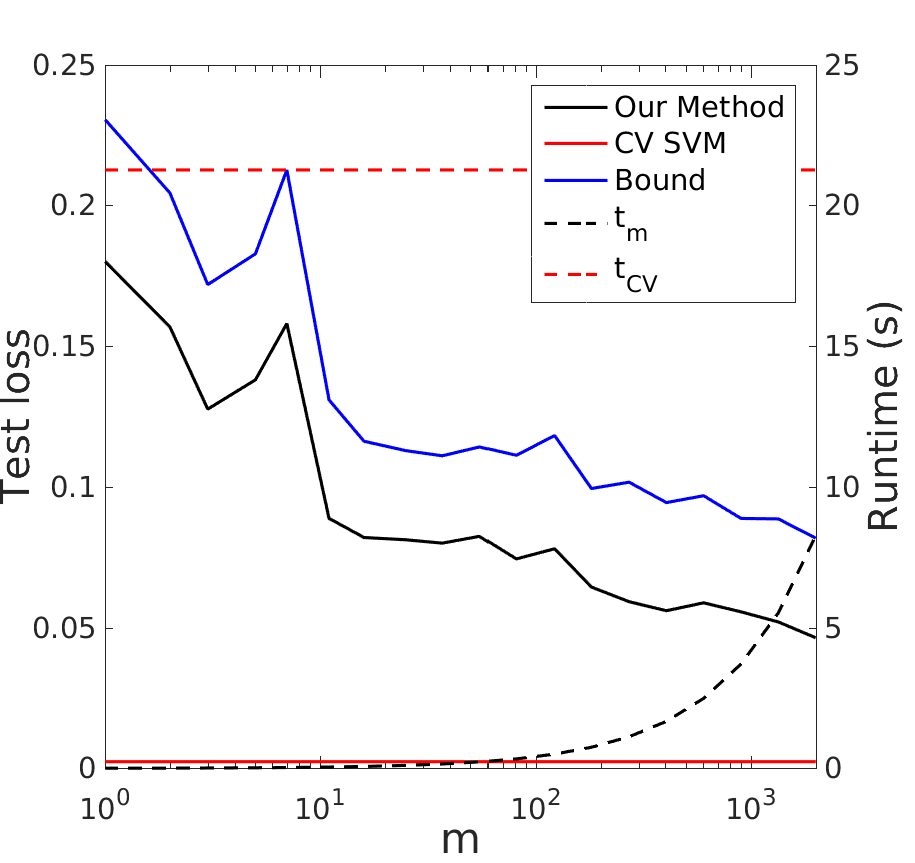}
}%
\subfigure[Mushrooms dataset. $n = 2000$,\newline$r = d+1 = 113$.]{
\includegraphics[width=0.5\textwidth]{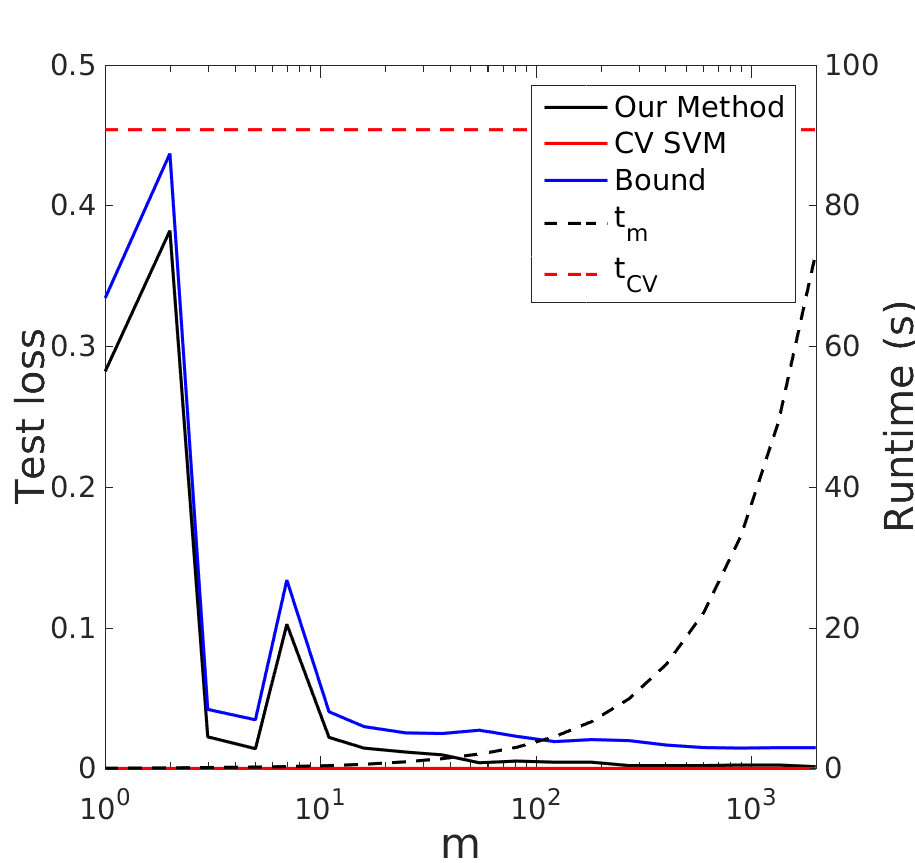}
}
\caption{\textbf{Comparison of PAC-Bayesian aggregation with RBF kernel SVM tuned by cross-validation.} The solid red, black, and blue lines correspond, respectively, to the zero-one test loss of the cross-validated SVM, the loss of $\rho$-weighted majority vote, where $\rho$ is a result of minimization of the PAC-Bayes-$\lambda$ bound, and PAC-Bayes-kl bound on the loss of randomized classifier defined by $\rho$. The dashed black line represents the training time of PAC-Bayesian aggregation, while the red dashed line represents the training time of cross-validated SVM. The prediction accuracy and run time of PAC-Bayesian aggregation are given as functions of the hypothesis set size $m$.}
\label{fig:appruntime}
\end{figure*}

\begin{figure*}
\centering
\subfigure[Breast cancer dataset.\newline$n = 340$, $r = d + 1 = 11$.]{
\includegraphics[width=0.37\textwidth]{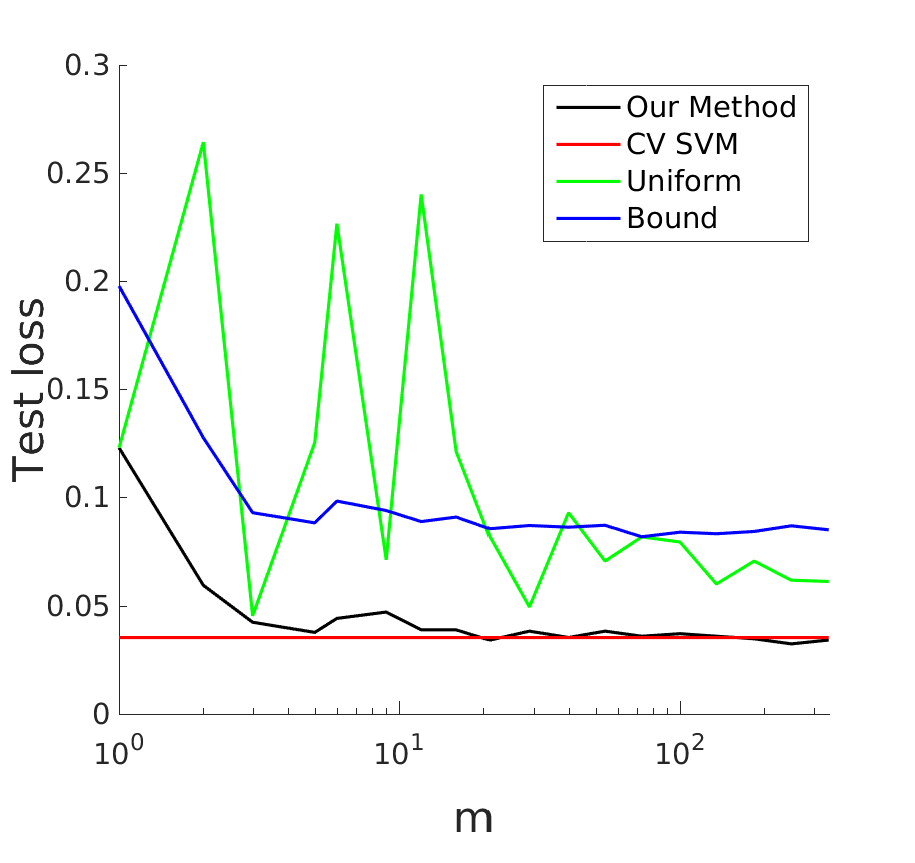}
}%
\qquad\qquad
\subfigure[Haberman dataset. $n = 150$, $r = \sqrt n = 12$.]{
\includegraphics[width=0.37\textwidth]{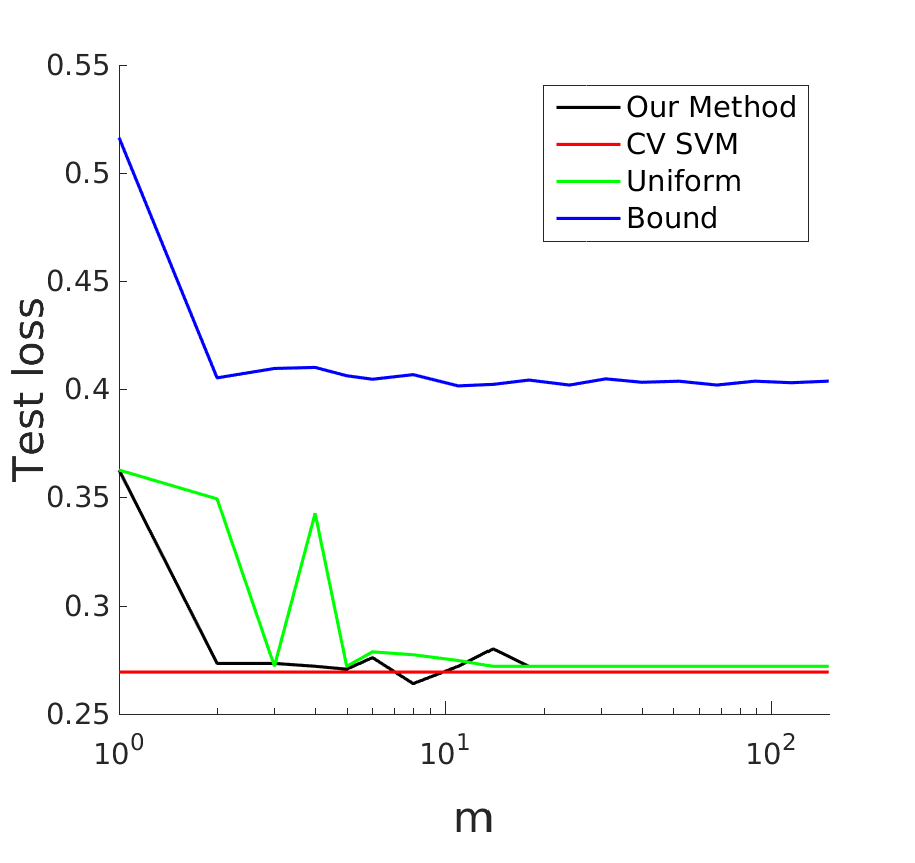}
}
\subfigure[AvsB dataset. $n = 1000$,\newline$r = d+1 = 17$.]{
\includegraphics[width=0.37\textwidth]{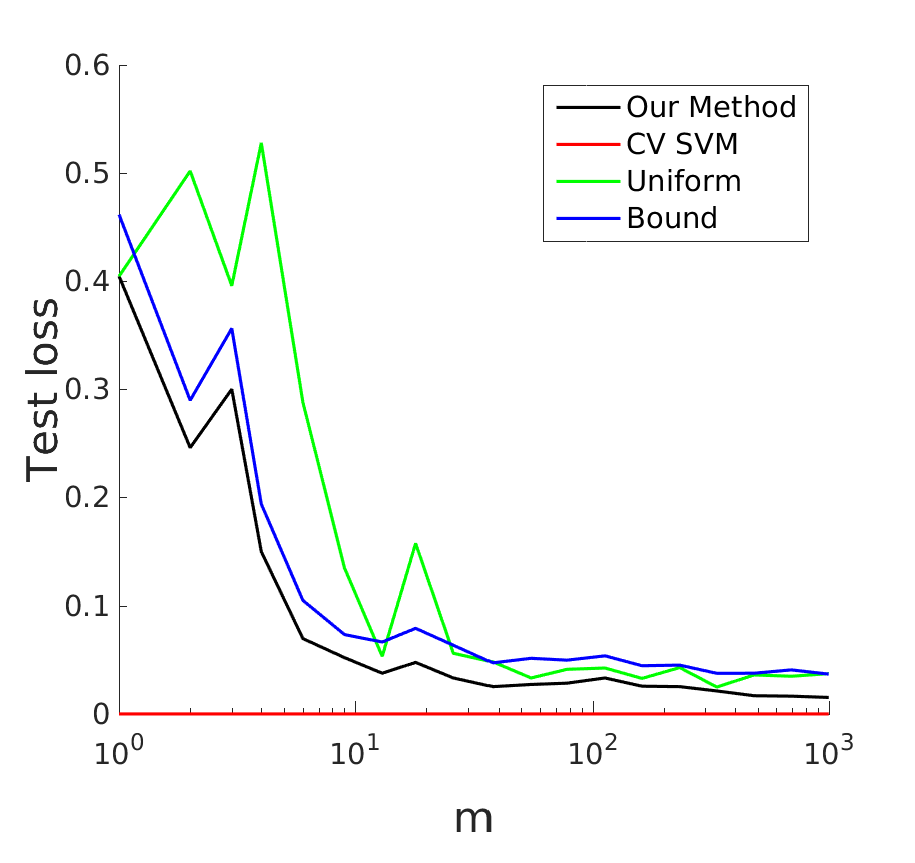}%
}%
\qquad\qquad
\subfigure[Waveform dataset. $n = 2000$, $r = d+1 = 41$.]{
\includegraphics[width=0.37\textwidth]{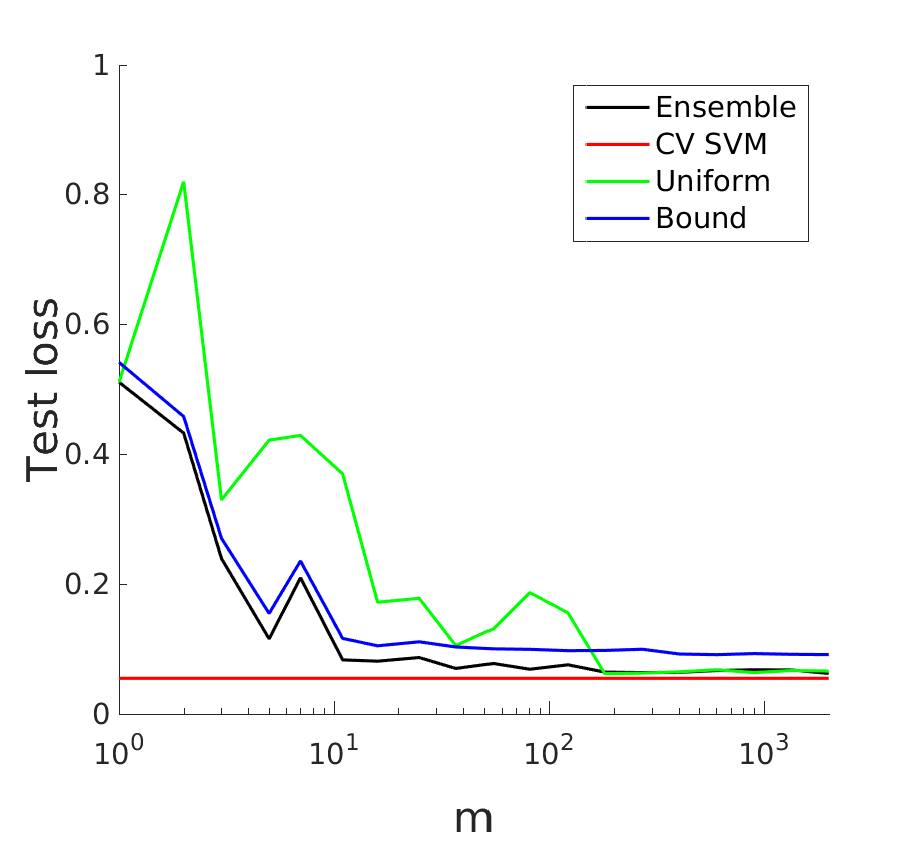}
}
\subfigure[Skin dataset. $n = 2000$,\newline$r = \sqrt n = 45$.]{
\includegraphics[width=0.37\textwidth]{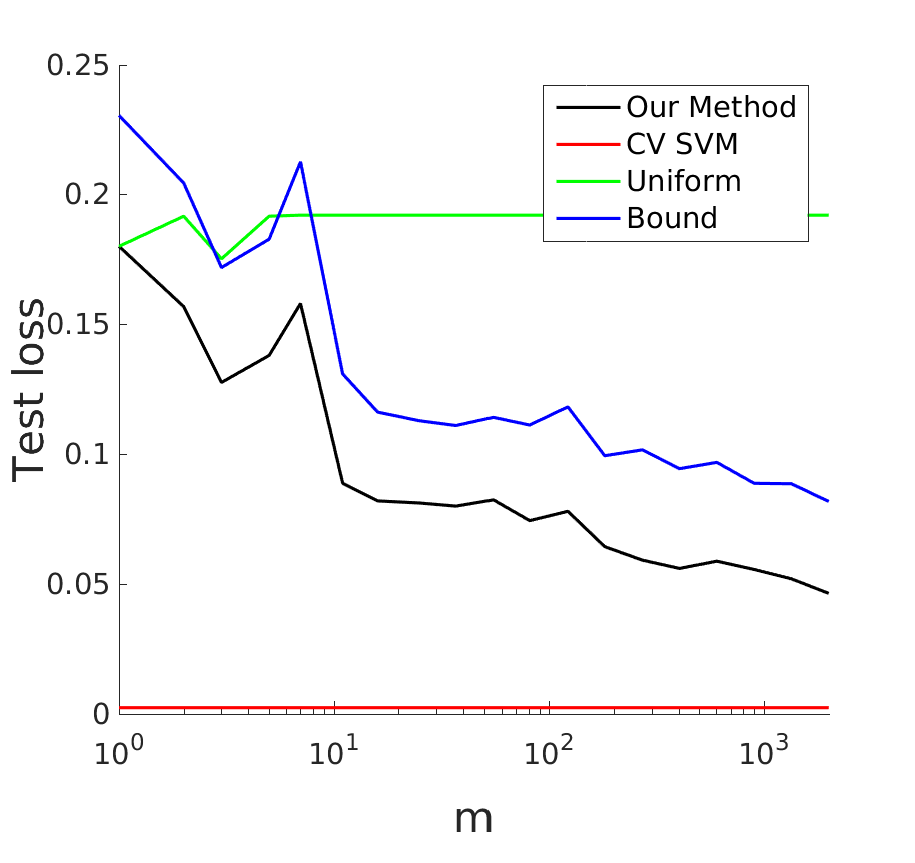}
}%
\qquad\qquad
\subfigure[Mushrooms dataset. $n = 2000$, $r = d+1 = 113$.]{
\includegraphics[width=0.37\textwidth]{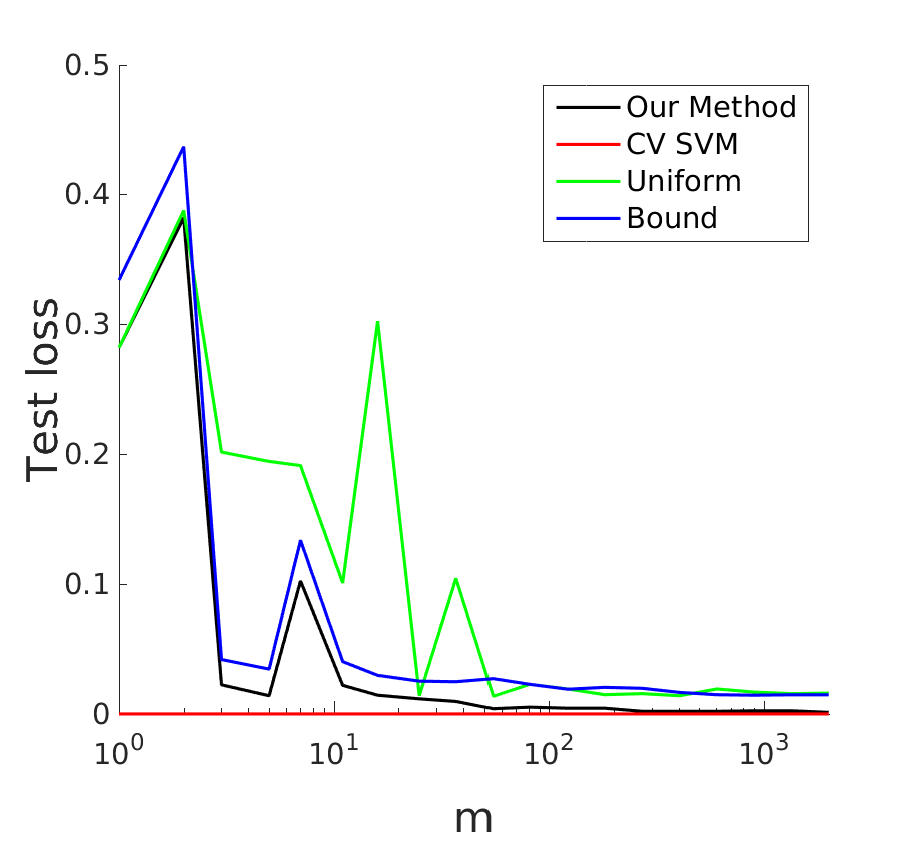}
}
\caption{\textbf{Prediction performance of $\rho$-weighted majority vote (``Our Method''), uniform majority vote (``Uniform''), and cross-validated SVM (``CV SVM'') together with the PAC-Bayes kl bound (``Bound'')}.}
\label{fig:appprediction}
\end{figure*}

\begin{figure*}
\centering
\subfigure[Ionosphere dataset. $n = 200$, $r = d+1 = 35$.]{
\includegraphics[width=0.5\textwidth]{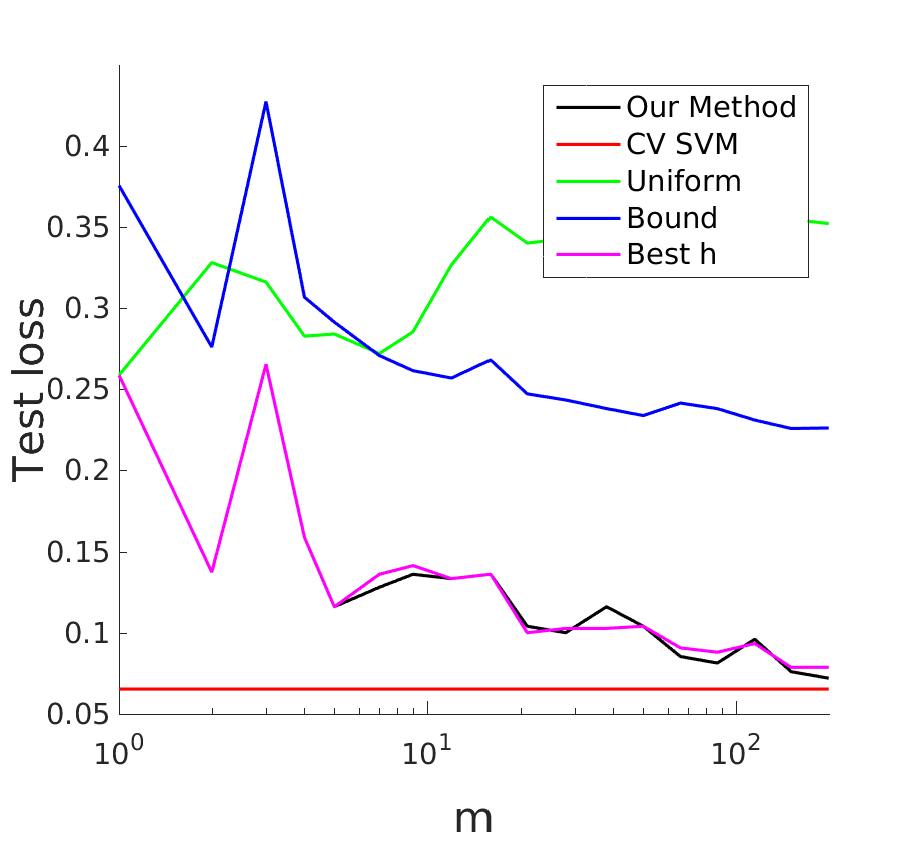}
}%
\subfigure[Adult dataset. $n = 2000$, $r = d+1 = 123$.]{
\includegraphics[width=0.5\textwidth]{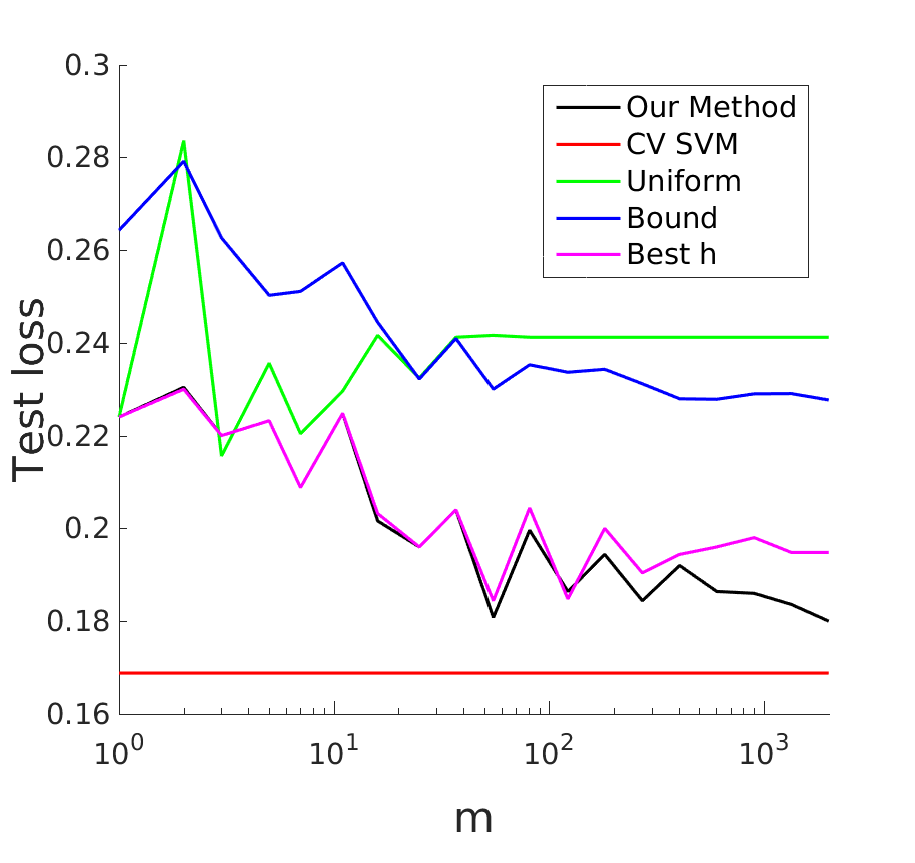}
}
\caption{\textbf{Prediction performance of weighted majority vote, uniform majority vote, SVM corresponding to the maximum of $\rho$ (Best $h$), and cross-validated SVM together with the PAC-Bayes-kl bound}.}
\label{fig:appprediction_max}
\end{figure*}

\subsection{Comparison of the Alternating Minimization with Grid Search for Selection of $\lambda$}

\begin{table}[t]
\centering
\footnotesize
\setlength{\tabcolsep}{.3em}
\begin{tabular}{|l|c|c|c|}
\hline
\textbf{Name} & $|\text{S}|$ & $|\text{V}|$ & $|\text{T}|$\\ \hline
Mushrooms & 2000 & 500 & 1000\\ \hline
Skin & 2000 & 500 & 1000 \\ \hline
Waveform & 2000 & 600 & 708 \\ \hline
Adult & 2000 & 500 & 685 \\ \hline
Ionosphere & 150 & 75 & 126 \\ \hline
AvsB & 700 & 500 & 355\\ \hline
Haberman & 150 & 50 & 106\\ \hline
Breast cancer & 300 & 100 & 283\\ \hline
\end{tabular}
\caption{\textbf{Sizes of dataset partitions used in Figure \ref{fig:direct_vs_validate}}. $|\text{S}|$ refers to the size of the training set, $|\text{V}|$ refers to the size of the validation set, and $|\text{T}|$ refers to the size of the test set.}
\label{tbl:uci_val}
\end{table}

In this section we present a comparison between direct minimization of the PAC-Bayes-$\lambda$ bound and selection of $\lambda$ from a grid using a validation set. Table \ref{tbl:uci_val} shows how each dataset is partitioned into training, validation, and test sets. The grid of $\lambda$-s was constructed by taking nine evenly spaced values in $[0.05; 1.9]$. For each $\lambda$ we evaluated on the validation set the performance of the majority vote weighted by the distribution $\rho(\lambda)$ defined in equation \eqref{eq:rho}, and picked the one with the lowest validation error. Note that the grid search had access to additional validation data that was unavailable to the alternating minimization procedure. Figure \ref{fig:direct_vs_validate} presents the results. We conclude that the bound minimization performed comparably to validation in our experiments.

\begin{figure*}
\centering
\includegraphics[width=\textwidth]{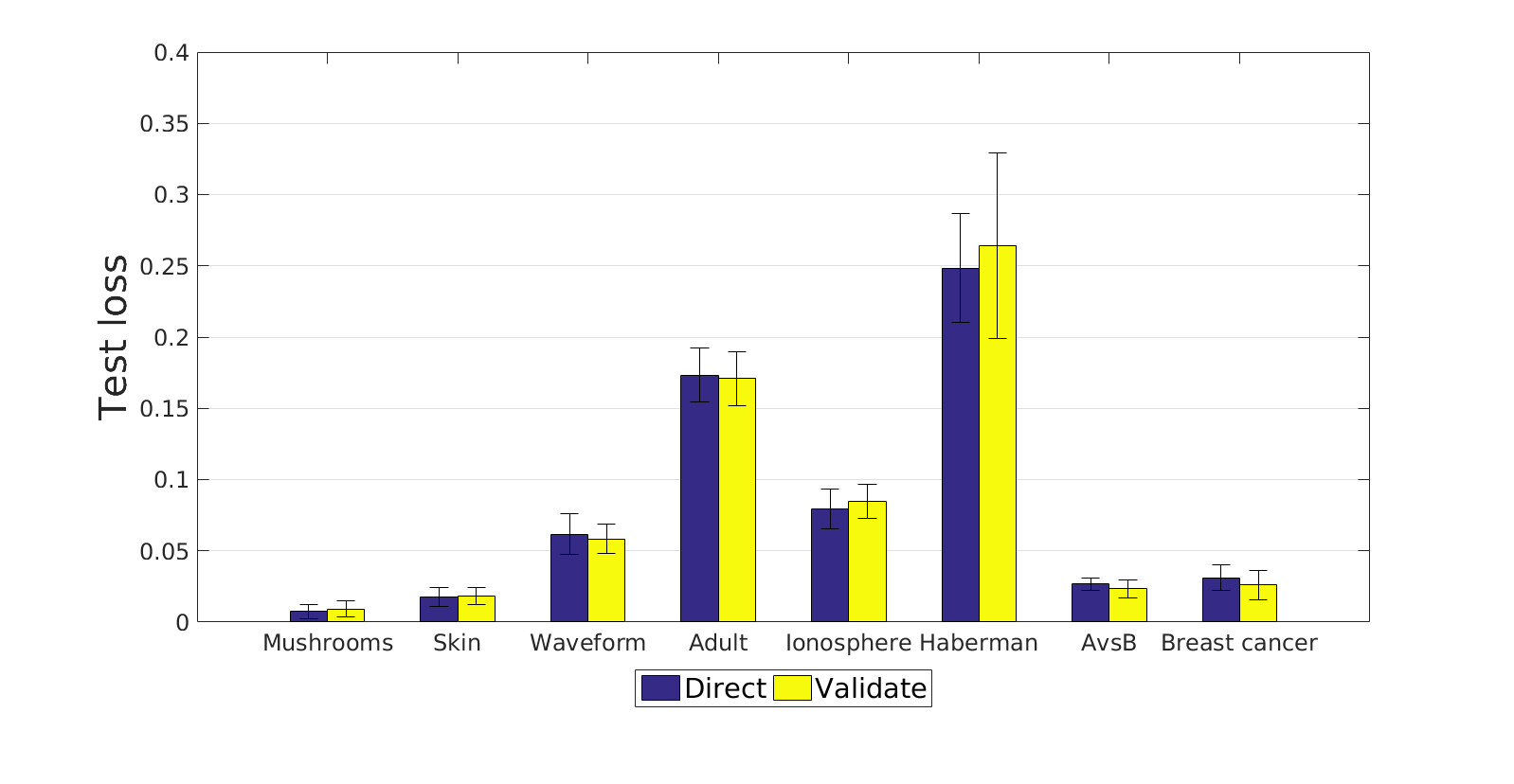}
\caption{\textbf{Comparison of the Alternating Minimization with Grid Search for Selection of $\lambda$.} We show the loss on the test set obtained by direct minimization of $\lambda$ (``Direct'') and grid search (``Validate''). Error bars correspond to one standard deviation over 5 splits of the data into training, validation, and test set.}
\label{fig:direct_vs_validate}
\end{figure*}

\subsection{Comparison of $\rho$-weighted Majority Vote with Randomized Classifier and Empirically Best Classifier}
\label{app:majority}

In this section we compare the performance of $\rho$-weighted majority vote with the performance of randomized classifier defined by $\rho$ and the performance of the best out of $m$ weak classifiers (measured by the validation loss). The comparison is provided in Figure~\ref{fig:mjv-vs-rnd}. Furthermore, in Table~\ref{tbl:50-pcnt} we provide the number of hypotheses that took up 50\% of the posterior mass $\rho$. While the performance of the randomized classifier is close to the performance of the best weak classifier, the distribution of posterior mass $\rho$ over several classifiers improves the generalization bound and reduces the risk of overfitting when $m$ is large. In other words, randomized classifier makes learning with large $m$ safer. In our experiments the majority vote provided slight, but not significant improvement over the randomized classifier.

\begin{table}[t]
\centering
\footnotesize
\setlength{\tabcolsep}{.3em}
\begin{tabular}{|l|c|}
\hline
\textbf{Name} & \#($h$) that make 50\% of $\rho$-mass\\ \hline
Mushrooms & 2\\ \hline
Skin & 1 \\ \hline
Waveform & 3 \\ \hline
Adult & 4 \\ \hline
Ionosphere & 2 \\ \hline
Haberman & 26\\ \hline
AvsB & 2\\ \hline
Breast cancer & 12\\ \hline
\end{tabular}
\caption{\textbf{The number of hypotheses that make up 50\% of the posterior mass $\rho$.}}
\label{tbl:50-pcnt}
\end{table}

\begin{figure*}[t]
\centering
\includegraphics[width=\textwidth]{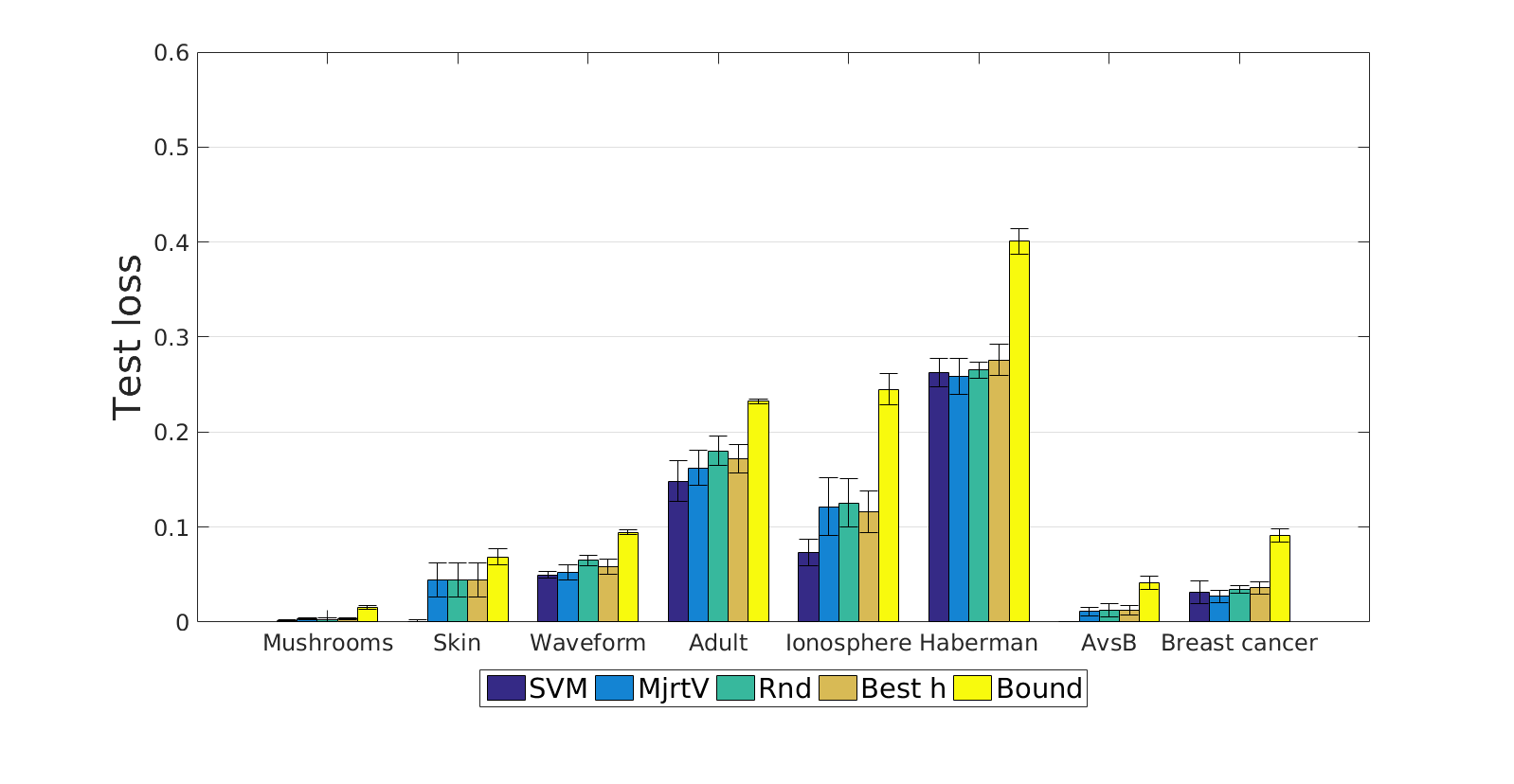}
\caption{\textbf{Comparison of cross-validated SVM (SVM), $\rho$-weighted majority vote (MjrtV), randomized classifier $\rho$ (Rnd), the best (empirically) out of $m$ weak classifiers (Best $h$), and the PAC-Bayesian bound (Bound).} The comparison is for the maximal value of $m$ ($m=n$) and the same values of $r$ as given in the main experiments in Figures~\ref{fig:runtime_kernel} and \ref{fig:appruntime}. Error bars correspond to one standard deviation over 5 splits of the data into training, validation, and test set.}
\label{fig:mjv-vs-rnd}
\end{figure*}

\end{document}